\documentclass[letterpaper]{article}
\usepackage[preprint]{aaai2027}
\usepackage[hyphens]{url}
\usepackage{graphicx}
\usepackage{natbib}
\usepackage{caption}
\usepackage{algorithm}
\usepackage{algorithmic}
\usepackage{amsmath}
\usepackage{amssymb}
\usepackage{amsthm}
\usepackage{booktabs}
\newtheorem{theorem}{Theorem}
\newtheorem{corollary}{Corollary}
\newtheorem{assumption}{Assumption}
\newcommand{\sg}{\mathrm{sg}}
\newcommand{\softmax}{\mathrm{softmax}}

\title{A Differentiable Atari VCS: A Complex, Fully Known\\
Ground Truth for Explainable AI}

\author{
    Andreas Maier\textsuperscript{\rm 1},\quad
    Siming Bayer\textsuperscript{\rm 1},\quad
    Patrick Krauss\textsuperscript{\rm 1,2}
}
\affiliations{
    \textsuperscript{\rm 1}Pattern Recognition Lab, Friedrich-Alexander-University Erlangen-Nuremberg, Germany\\
    \textsuperscript{\rm 2}Mannheim Center for Neuromodulation and Neuroprosthetics, Heidelberg University, Germany
}

\begin{document}
\maketitle

\begin{abstract}
Explanation requires ground truth: to verify that an account of a system is
correct, we must know the system's inner functioning. This is precisely what
is missing wherever explainable AI (XAI) is most needed. The systems we can
study today fall into two camps. Either they are simple and procedural---
decision trees, rule lists, sparse linear models---where the mechanism is
known but trivial, so explaining it tests nothing, or they are genuinely
complex---deep networks, real-world tasks---where XAI is needed but no ground
truth of the inner functioning exists, so an explanation can be plausible,
confident, and wrong with no way to tell. We set out to remove this
dichotomy by building the foundation for a study object that is at once
genuinely complex yet fully specified---inspectable by construction---and, so that gradient-based methods can
apply, fully differentiable. We re-implement the Atari 2600 Video Computer
System (VCS)---a real computer architecture, and the platform on which deep
reinforcement learning was first established---as two independent, end-to-end
\emph{differentiable} emulators, one in Julia (\textsc{jutari}) and one in
JAX (\textsc{jaxtari}), each validated bit-for-bit against the \textsc{xitari}
reference. Both ports reproduce \textsc{xitari} exactly on all 64 supported
Arcade Learning Environment (ALE) games: 64/64 byte-identical RAM and
64/64 pixel-identical screens. Treating the cartridge ROM as a weight tensor,
the RAM as a soft tape, and control flow as gates, we prove that the
differentiable (soft) execution is bit-exact equal to the original (hard)
execution in the forward pass at any finite temperature, while exposing
\emph{surrogate} gradients where the bit logic itself has none. The JAX port also
opens a GPU path: batched, fully differentiable rollouts reach millions of
environment-steps per second on a single commodity GPU. The system was built in roughly 137 hours
of active work over 29 calendar days, with large parts written autonomously
by coding agents. This paper builds and validates the foundation, and shows---
theoretically and in a qualitative gradient study---that gradient-based XAI on
it is feasible. The full code of both ports is available under the MIT license at \url{https://github.com/akmaier/UnderstandingVCS}.
\end{abstract}

\section{Introduction}

When we explain a system---in explainable AI (XAI) as anywhere else---we claim
that \emph{this} part is responsible for \emph{that} behaviour. Checking such a
claim requires knowing the inner functioning independently, as ground truth.
Without it we can ask whether an explanation is plausible or convincing, but not
whether it is \emph{true}. Ground truth is the prerequisite for verification, and
it is exactly what is missing where explanation matters most.

The systems available to study fall into two unsatisfying camps. Simple,
procedural models---decision trees, rule lists, sparse linear models---have fully
known inner functioning, but there the explanation is essentially the model
itself, testing an XAI method on nothing it could get wrong. These are not the
systems XAI was created for. The systems XAI does target---deep neural networks,
learned agents, real-world pipelines---are genuinely complex and, for that very
reason, have no ground truth of their inner functioning, so an attribution or
saliency map can be plausible, confident, and wrong with no way to tell
\citep{atrey2020exploratory,nikulin2019freelunch}.

Neuroscience faces a similar problem. Therefore, \citet{jonas2017could} asked
whether the standard systems-neuroscience toolkit---connectomics, tuning curves,
lesioning, dimensionality reduction---could recover how a MOS\,6502 microprocessor
computes, using the chip as a model organism precisely because its mechanism is
complex \emph{and} completely known. Even with perfect, unlimited data the methods
produced structure that looked meaningful but did not recover the processor's
actual organisation. The conclusion was that the \emph{methods} were lacking, and
that the way to find out is to test them on a system whose ground truth we already
possess. Two things make it difficult to translate this to modern XAI: the
microprocessor was not \emph{differentiable}, so today's gradient-based methods
could not be applied, and the analyses were classical statistics, not XAI. The
MOS\,6502 is no accident---it is the CPU at the heart of the Atari 2600 Video
Computer System (VCS), the platform on which modern deep reinforcement learning
was first demonstrated \citep{mnih2013playing,mnih2015human,bellemare2013arcade}.

We therefore build the missing piece: a study object at once \emph{complex} (a
real computer architecture), \emph{fully specified} to us bit-for-bit, and
\emph{differentiable}, so the gradient-based XAI toolkit can in principle be turned
on it. We re-implement the VCS---a 6507 CPU executing code from a cartridge ROM, a
Television Interface Adapter (TIA) that turns register writes into pixels cycle by
cycle, and a RIOT (RAM, I/O, timer) chip providing the RAM, an interval timer, and
joystick and switch I/O---as two independent, end-to-end differentiable emulators:
\textsc{jutari} in Julia \citep{bezanson2017julia} and \textsc{jaxtari} in JAX
\citep{bradbury2018jax}. Each runs in a bit-exact \emph{hard} mode and a
differentiable \emph{soft} mode, validated against \textsc{xitari} \citep{xitari},
the Stella-derived C\texttt{++} emulator \citep{stella} that served as the
reference for the original deep-Q-network (DQN) work.

This paper builds and validates that foundation, and shows---both theoretically
and in a qualitative gradient study---that XAI on this system is feasible. Our
contributions are:
\begin{itemize}
  \item \textbf{Two differentiable VCS ports.} Independent Julia
  (\textsc{jutari}) and JAX (\textsc{jaxtari}) implementations of the full VCS
  (CPU, TIA, RIOT, cartridge bank-switching, console, controllers), each with a
  bit-exact hard path and a differentiable soft path. Both ports are
  bit-for-bit identical to \textsc{xitari} on all 64 supported games---64/64
  byte-identical RAM and 64/64 pixel-identical screens.
  \item \textbf{A soft/hard formulation with theoretical analysis thereof.}
  We treat the ROM as a weight tensor, the RAM as a soft tape, and control
  flow as gates, and show that the soft forward pass is bit-exact equal to the
  hard forward pass at any finite temperature (Theorem~\ref{thm:exact}), and
  give the temperature-limit error bound for the fully relaxed variant
  (Theorem~\ref{thm:limit}).
  \item \textbf{An AI-assisted engineering account.} A quantified report of
  how an emulator port that would classically cost many person-months was
  built in roughly 137 hours of active work, measured from the version-control
  log, with large parts written autonomously by coding agents.
  \item \textbf{A conformance evaluation} across the 64-game Arcade Learning
  Environment (ALE) set, with
  formally defined bit- and pixel-exactness metrics, and the observation that
  a bit-exact re-implementation is itself an audit tool for its reference.
  \item \textbf{A qualitative gradient study} showing that the foundation works
  as intended: the forward pass is bit-exact and the backward pass supplies
  \emph{surrogate} gradients with respect to registers, RAM, and ROM bytes, so
  attributions can be checked against the known mechanism---demonstrated here on
  the joystick-to-screen control path.
\end{itemize}

\section{Background and Related Work}

The dominant XAI tools are \emph{attribution} or \emph{saliency} methods,
which assign each input element a number reflecting how much it mattered to
the output. Grad-CAM computes the gradient of a target output with respect
to a convolutional layer and renders it as a heatmap of important regions
\citep{selvaraju2017gradcam}. Grad-CAM++ refines this with a pixel-wise
weighting of positive gradients \citep{chattopadhyay2018gradcampp}. These
methods are powerful and architecture-agnostic, but they share a limitation
their own authors flag: there is no way to confirm that a heatmap reflects
the network's true reason for deciding, because for a deep network that
reason is unknown. Grad-CAM++ notes explicitly that conventional
faithfulness metrics ``need not correlate with the actual factors
responsible for the network's decision''
\citep{chattopadhyay2018gradcampp}. Validation therefore falls back on
proxies---bounding-box overlap, or human-trust studies.

The problem sharpens when the object of explanation is a reinforcement
learning (RL) agent rather than an image classifier. A widely studied such
agent is the DQN, which learns to play Atari games from raw pixels
\citep{mnih2013playing,mnih2015human} and became a standard testbed.
\citet{greydanus2018visualizing} introduced a perturbation-based
saliency for Atari agents and described strategies such as Breakout
``tunneling''. The one case in which they verify the saliency is an
experiment where they inserted the ground-truth feature themselves.
\citet{nikulin2019freelunch} built saliency directly
into the agent and validated it against human eye-tracking---a behavioural
surrogate, not the machine's own computation. \citet{atrey2020exploratory}
then made the consequence explicit: using counterfactual interventions (for
example, mirroring the Breakout wall) they showed that popular
saliency-based explanations of Atari agents are frequently unfalsifiable
and do not hold up, and concluded that saliency maps are best treated as an
\emph{exploratory}, not an \emph{explanatory}, tool. \citet{such2019atarizoo}
industrialised the comparison of trained agents while leaving the validation
gap untouched.

Surveys of explainable RL confirm that this is the field's structural
condition rather than an incidental gap
\citep{qing2022xrlsurvey,vouros2023xrlsurvey,cheng2025xrlsurvey,saulieres2025xrlsurvey}.
Across hundreds of methods the target is uniformly described as a closed or
black box, and objective, mechanism-grounded evaluation is repeatedly named
among the field's open needs. Even methods designed for interpretability
inherit the problem: XDQN makes a DQN interpretable by training a transparent
surrogate to mimic it, but its strongest correctness measure is fidelity
\emph{to the DQN}---agreement with another black box, not with a known
mechanism \citep{kontogiannis2023xdqn}. The recurring obstacle is the one
\citet{jonas2017could} identified: without a study object whose true mechanism
is available, an explanation cannot be checked. We build such an
object---complex, fully specified, and differentiable---and validate it, opening
the way to develop and validate XAI tools against a known mechanism.

\subsection{Known Operators and Differentiable Programming}

Our approach is related to \emph{known-operator learning}. The idea began with
deep-learning computed tomography \citep{wuerfl2016deep}, where a known
reconstruction operator is built into the network rather than learned from
data. \citet{maier2019known} later showed that embedding known, differentiable
operators directly into a network reduces its maximum error bound: if a layer
is known, its error term vanishes and the corresponding part of the bound
cancels, even for networks of arbitrary depth. The principle is general---``any operation that allows
computation of a gradient or sub-gradient towards its inputs'' can be
embedded \citep{maier2019known}---and the same spirit underlies
physics-informed neural networks, which embed known governing equations
\citep{raissi2019pinn}, and operator-aware treatments of deep
learning for practitioners \citep{maier2019gentle}. Our VCS is, in this
sense, an extreme case: \emph{every} operator is known, so the entire
forward computation is a fixed, differentiable program, and the only thing
left to ``explain'' is how that known program maps inputs to outputs.

Making a hand-written emulator differentiable requires a general-purpose
differentiable-programming substrate. Julia provides source-to-source
automatic differentiation (AD) over essentially arbitrary code via Zygote
\citep{innes2019zygote}, with eager execution and mutable state that map
naturally onto an emulator's registers and RAM. JAX provides trace-based AD
with just-in-time compilation to XLA \citep{bradbury2018jax}, which favours
pure, functional, array-shaped code. The two make different trade-offs for
this workload, and building both lets us cross-check each port against the
other in addition to the reference. We treat the cartridge ROM as a weight
tensor, the 128 bytes of RAM as a Neural-Turing-Machine-style addressable
tape \citep{graves2014ntm}, and discrete control flow as gates relaxed with
the Gumbel-softmax reparameterisation \citep{jang2017gumbel} and
straight-through estimators \citep{bengio2013estimating}.

\subsection{Why the Atari VCS Is Good Ground Truth}

\begin{figure}[t]
\centering
\includegraphics[width=0.92\columnwidth]{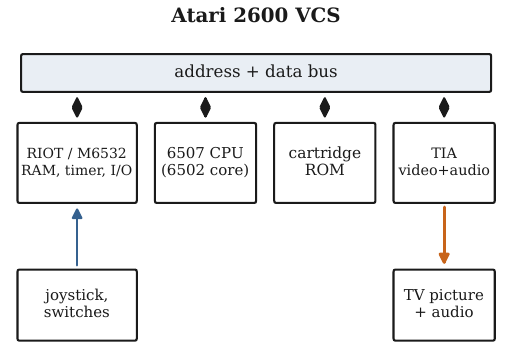}
\caption{The Atari 2600 VCS. A 6507 CPU executes code from a
(bank-switched) cartridge ROM over a shared address/data bus. The RIOT
provides 128 bytes of RAM, a timer, and joystick/switch I/O, and the TIA turns
register writes into a picture and sound, cycle by cycle. Every block is
re-implemented and differentiated in both ports.}
\label{fig:arch}
\end{figure}

The VCS (Figure~\ref{fig:arch}) is a real computer architecture, small
enough to know completely yet complex enough to be interesting. A 6507 CPU
(a cost-reduced MOS\,6502) executes a program stored in the cartridge ROM.
Because the address space is only 13 bits wide, larger cartridges
\emph{bank-switch}: writes to special addresses swap which 4\,KB ROM page is
visible. The RIOT chip (M6532) holds the 128 bytes of system RAM, an
interval timer, and the input ports for the joystick, paddles, and console
switches. The heart of the machine is the TIA: it has no frame buffer, so
the CPU must race the electron beam, writing its registers in time with each
scanline. These include the colour registers---\texttt{COLUP0}
and \texttt{COLUP1} for the two player sprites, \texttt{COLUPF} for the
playfield, and \texttt{COLUBK} for the background---together with sprite-shape
and position registers. The screen and audio are the TIA's outputs. This is
the system the Arcade Learning Environment \citep{bellemare2013arcade} exposes
to RL agents, and that \textsc{xitari} \citep{xitari}, a fork of the Stella
emulator \citep{stella}, uses as the reference for DQN. Because \textsc{xitari}
is deterministic given a ROM, an action stream, and a seed, every register,
RAM byte, and pixel our ports produce can be compared against it, and any
disagreement is a measurable mistake in the software port.

\paragraph{Other Atari substrates and differentiable simulators.}
The Atari platform is also used as a throughput target and as a structured RL
environment. CuLE ports ALE to CUDA for massive-batch rollouts (up to
${\sim}155$M frames/h on a single GPU) but is neither bit-exact nor
differentiable \citep{dalton2020cule}; object-centric and JAX-native Atari
environments such as OCAtari and JAXAtari expose structured states and run on
the GPU \citep{delfosse2023ocatari,delfosse2025jaxatari}, but JAXAtari
\emph{reimplements} each game's logic in JAX and cannot execute the original ROMs
at all, and both are RL \emph{environments} rather than a gate-level
differentiable port validated bit-for-bit against the emulator reference. Differentiable simulators are by now standard in
continuous physics---Brax back-propagates through rigid-body dynamics in JAX
\citep{freeman2021brax}---whereas our substrate differentiates a \emph{discrete}
computer architecture whose hard forward pass is provably exact. We follow the
ALE evaluation protocol and its game catalogue
\citep{bellemare2013arcade,machado2018revisiting}.

\section{Building Two Differentiable VCS Ports}

We ported \textsc{xitari} subsystem by subsystem---CPU, bus, TIA, RIOT,
cartridge mappers, console, controllers, and the ALE-style environment
wrapper---under a strict conformance gate. We define ``correct'' by three
automated test harnesses: programs that drive a port and the reference on
identical inputs and compare their internal state byte for byte, failing on
any mismatch. The first (PXC1) compares each port against a per-instruction
\textsc{xitari} trace of the CPU registers, RAM, and TIA state. The second
(PXC2) cross-checks the two ports against each other, so that a bug shared by
the reference and one port is still caught by the second, independent
implementation. The third (PXC-S) compares the full $210\times160$ frame
buffer. Each port implements all 151 documented NMOS\,6502 instructions
plus the undocumented \texttt{USBC} alias and 37 common ``illegal''
opcodes (NOP/LAX/SAX families)---189 opcode bytes in total---in both
execution modes, and nine cartridge mappers (2K, 4K, F8, F6, F4, their
Superchip variants, and E0) with content-based auto-detection mirroring
the reference. Table~\ref{tab:ports} summarises the two ports.

\begin{table}[t]
\centering
\setlength{\tabcolsep}{4pt}
\caption{The two differentiable ports at a glance. Both reproduce
\textsc{xitari} bit-for-bit on all 64 games. Throughput rows are soft-mode
env-steps/s (Pong, $3{,}000$ steps); a dash marks a regime the port does not
target. Full per-batch scaling is in the supplementary material.}
\label{tab:ports}
\begin{tabular}{lcc}
\toprule
 & \textsc{jutari} & \textsc{jaxtari} \\
\midrule
Language / AD            & Julia / Zygote & JAX / XLA \\
Source lines            & 8{,}630 & 10{,}111 \\
Test lines              & 6{,}052 & 11{,}855 \\
Test cases              & ${\sim}1{,}183$ & 803 \\
Opcode bytes            & 189 & 189 \\
Cartridge mappers       & 9 & 9 \\
RAM-exact (of 64)       & \textbf{64} & \textbf{64} \\
Pixel-exact (of 64)     & \textbf{64} & \textbf{64} \\
\addlinespace
Throughput, 1 env (CPU) & $\mathbf{370{,}100}$ & $1{,}178$ \\
Throughput, batched CPU & --- & ${\sim}60{,}000$ \\
Throughput, batched GPU & --- & \textbf{$\sim$3.1\,M} \\
\bottomrule
\end{tabular}
\end{table}

\subsection{Hard and Soft Execution}

\begin{figure*}[t]
\centering
\includegraphics[width=0.90\textwidth]{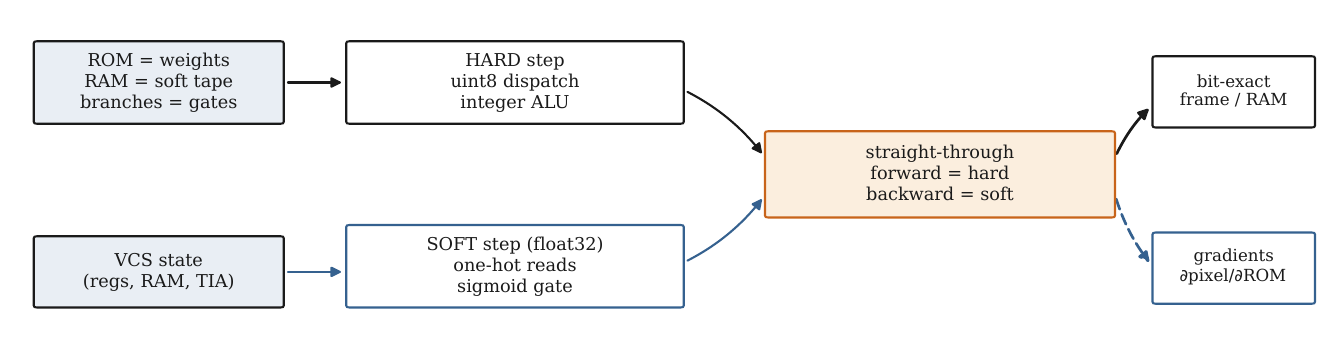}
\caption{The two execution paths share one known mechanism. The HARD step
uses integer dispatch and exact bit logic. The SOFT step reads the ROM and
RAM as one-hot dot products and relaxes the branch decision with a sigmoid
gate. A straight-through estimator joins them: the forward value is the
hard one (bit-exact), while the backward pass uses the soft gradient, so the
substrate emits an exact frame together with a \emph{surrogate} gradient of each
pixel with respect to ROM bytes and registers (forward exact; backward from the
relaxation).}
\label{fig:pipeline}
\end{figure*}

Each port runs in two modes that share one code path (Figure~\ref{fig:pipeline}).
In \textbf{hard} mode, the emulator is the ordinary bit-exact machine:
opcode bytes index an integer dispatch table, the ALU is integer
arithmetic, and flags are bits. In \textbf{soft} mode, the same step is
re-expressed so that gradients can flow. Three relaxations carry the idea.

First, every ROM and RAM read is written as a dot product against a one-hot
address vector. For a memory tensor $r\in\mathbb{R}^{M}$ of size $M$ (the ROM or
RAM) and an address $a\in\{0,\dots,M-1\}$,
\begin{equation}
\mathrm{peek}(r,a) \;=\; \mathbf{1}_a^{\top} r \;=\; \sum_{i} [\![i=a]\!]\, r_i \;=\; r_a .
\label{eq:peek}
\end{equation}
The forward value is exactly $r_a$, but the gradient
$\partial\,\mathrm{peek}/\partial r = \mathbf{1}_a$ is now defined and
one-hot---this is what makes ``which ROM byte explains this pixel?'' a
gradient question. The RAM tape uses the identical construction, which is
the discrete limit of the soft, attention-style addressing of a Neural
Turing Machine \citep{graves2014ntm}.

Second, opcode dispatch is, in principle, a convex combination over the
handler outputs. Let $\ell\in\mathbb{R}^{K}$ be a score vector over the
$K=256$ possible opcodes---the \emph{logits}. In the hard machine $\ell$ is a
one-hot indicator of the decoded opcode, and in the relaxed form it is a soft
score. With temperature $T>0$ and softmax weights $w=\softmax(\ell/T)$, that
is $w_k=e^{\ell_k/T}/\sum_j e^{\ell_j/T}$, the dispatch output is
\begin{equation}
\mathrm{select}(\ell,V;T) = w^{\top} V,
\label{eq:select}
\end{equation}
where the rows of $V$ are the candidate handler outputs. This collapses to the
hard pick as $T\to0$ \citep{jang2017gumbel}. In the executed step we use the
saturated form directly---a hard switch on the decoded opcode---so the forward
pass is exact, and the relaxed form~\eqref{eq:select} is what a fully soft variant
would use. The same softmax select governs \emph{any} discrete choice, with $T$
setting how widely the soft pick spreads---and with it the range over which
gradients flow. Figure~\ref{fig:tempheat} shows this on a visualisable selection
(placing a target sprite among candidate screen columns rather than choosing an
opcode): raising $T$ spreads the pick over neighbouring columns, widening the
\emph{capture range} over which gradients reach the input, while the forward render
stays exact.

\begin{figure}[t]
\centering
\includegraphics[width=\columnwidth]{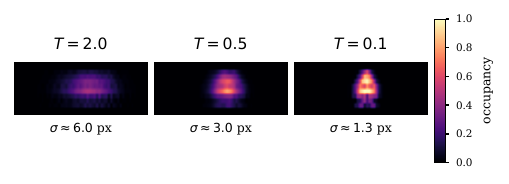}
\caption{Soft-select temperature $T$ in pixel space, on a visualisable selection:
a target sprite's screen column chosen with $w=\softmax(\ell/T)$. For each $T$ we
sample the column $100$ times, render the hard sprite, and average, so each pixel
is the expected occupancy. Low $T$ is a single sharp sprite (the hard pick,
$\sigma\approx1.3$ px). Raising $T$ spreads it over neighbouring columns
($\sigma\approx3$, then $6$ px), widening the range over which gradients reach---a
knob for the gradient's capture range. The forward pass stays bit-exact
(Theorem~\ref{thm:exact}).}
\label{fig:tempheat}
\end{figure}

Third, a conditional branch---``if the status flag is set, add the offset
$\delta$ to the program counter (\textsc{pc})''---is relaxed with a sigmoid
gate. The branch condition is one processor status flag, a single bit in the
hard machine. In soft mode we let it take a real value $f\in[0,1]$, a relaxed
flag. Writing $f$ as a logit $z = s\,(2f-1)$, where the sign $s$ selects
branch-when-set ($s={+}1$) or branch-when-clear ($s={-}1$), and with sharpness
$\alpha$, the gate $g$ and relaxed program counter are
\begin{equation}
g = \sigma(\alpha z),\qquad
\mathrm{pc}_{\mathrm{soft}} = (1-g)\,\mathrm{pc}_{\bar b} + g\,\mathrm{pc}_{b}
= \mathrm{pc}_{\bar b} + g\,\delta,
\label{eq:branch}
\end{equation}
where $\mathrm{pc}_{\bar b}$ is the fall-through (not-taken) counter and
$\mathrm{pc}_{b}=\mathrm{pc}_{\bar b}+\delta$ is the taken counter, so the blend
simplifies to $\mathrm{pc}_{\bar b}+g\,\delta$. As $\alpha\to\infty$ the gate
saturates to a hard branch. For the fully relaxed variant (no straight-through
correction) $\alpha$ must be large enough that the gate's residual uncertainty
$g(1-g)$ stays below the integer-rounding threshold of $\mathrm{pc}$, so the
rounded counter still matches the hard branch (see the supplementary material).

\begin{figure}[t]
\centering
\includegraphics[width=\columnwidth]{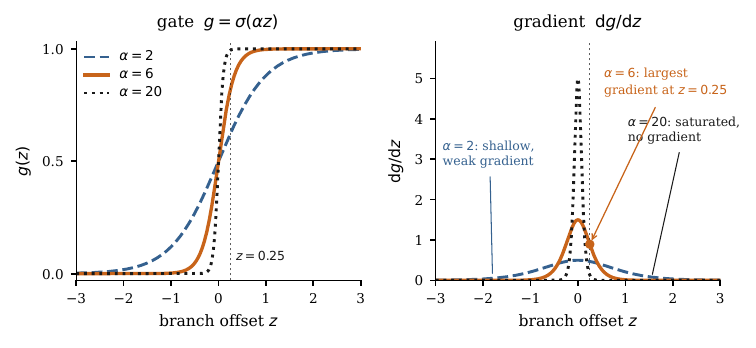}
\caption{The soft-branch gate $g=\sigma(\alpha z)$ (left) and its gradient
$\mathrm{d}g/\mathrm{d}z=\alpha\,\sigma(\alpha z)(1-\sigma(\alpha z))$ (right) for
$\alpha\in\{2,6,20\}$, with the operating point $z{=}0.25$ marked. Too large an
$\alpha$ (here $20$) saturates the gate, leaving essentially no gradient away from
the switch. Too small an $\alpha$ (here $2$) gives a shallow gate whose gradient is
weak and barely varies with $z$. The intermediate $\alpha{=}6$ places the operating
point in the sigmoid's steep band and yields the largest, most localised gradient.}
\label{fig:sigmoid}
\end{figure}

These relaxations are connected to the exact machine by a
\emph{straight-through estimator} (STE) \citep{bengio2013estimating}. For any
soft quantity with a matching hard value,
\begin{equation}
\mathrm{STE}(\text{soft},\text{hard}) = \text{soft} + \sg(\text{hard}-\text{soft}),
\label{eq:ste}
\end{equation}
where $\sg(\cdot)$ is the stop-gradient operator. The forward value is
exactly \textit{hard}, and the backward pass uses the derivative of the soft
value. The
round and clamp operations are treated the same way (forward exact, backward
identity inside the valid range). The consequence, made precise next, is that
soft mode reproduces the hard machine \emph{exactly} in the forward pass and
differs only in the gradient it exposes.

This is the same device that makes \emph{max-pooling} differentiable, and it
applies far beyond branches. Every hard decision in the machine is a discrete
\emph{selection}: which opcode handler runs, whether a branch is taken, which
object owns a pixel under TIA priority, which sprite bit lights it. Each is a
switch whose winner we record in the forward pass, and the backward pass then
routes the gradient to the selected value, exactly as max-pooling routes its
gradient to the argmax input. With the switch stored, the entire \emph{content}
path---register values, sprite colours and graphics bits, priority
compositing---is differentiable with no approximation and a bit-exact forward
pass. The one quantity a stored switch cannot differentiate is a discrete
\emph{index}: the column at which a sprite is placed by strobe timing, for the
same reason max-pooling gives no gradient with respect to the pooling location.
There, and only there, a differentiable sampler---a sub-pixel triangular
(bilinear) kernel, as in spatial transformer networks
\citep{jaderberg2015spatial}---restores the position gradient.

\section{Soft Equals Hard: Equivalence and Gradients}

We now state what the construction guarantees. Let $\Phi_{\mathrm{H}}$ and
$\Phi_{\mathrm{S}}$ be the one-step transition maps of the hard and soft
emulators, acting on the full state $x$ (CPU registers, RAM, RIOT, TIA, and
frame buffer). We state both theorems here and give only \emph{proof sketches}.
The complete, formal proofs are in the supplementary material. The first result
is that the two maps agree in value.

\begin{theorem}[Exact forward equivalence]
\label{thm:exact}
For every reachable state $x$ and every finite sharpness $\alpha$ and
temperature $T$, the soft and hard one-step maps are equal by design,
$\Phi_{\mathrm{S}}(x) = \Phi_{\mathrm{H}}(x)$, with identical float32
representations. Consequently, over a trajectory of $N$ steps,
$x^{\mathrm{S}}_t = x^{\mathrm{H}}_t$ for all $t\le N$. The modes differ only in
the Jacobian
$\partial\Phi_{\mathrm{S}}/\partial x$, which is defined and generically
nonzero where $\partial\Phi_{\mathrm{H}}/\partial x$ is zero or undefined.
\end{theorem}

\begin{proof}[Proof sketch]
Each handler's forward output composes one-hot ROM/RAM reads~\eqref{eq:peek},
integer/round arithmetic, a hard opcode dispatch, and the straight-through
branch~\eqref{eq:ste}, each forward-identical to the hard handler, and the
stop-gradient alters only the backward rule. The one-step equality then extends
to whole trajectories by induction on the step index. The complete proof, with
the per-primitive equalities and the formal induction, is in the supplementary
material.
\end{proof}

By Theorem~\ref{thm:exact}, attribution in soft mode is computed on the exact
hard trajectory rather than an approximation, and the forward error is
identically zero independent of the temperature. The remaining question is the
behaviour of a \emph{fully} relaxed variant that does not apply the
straight-through correction. For that we need a mild non-degeneracy assumption.

\begin{assumption}[Decision margins]
\label{ass:dm}
Along the executed trajectory of length $N$, every discrete decision has a
strictly positive margin. For opcode dispatch at step $t$ with logits
$\ell^{(t)}$ and winner $k^\star_t$, the logit gap
$\Delta_t = \ell^{(t)}_{k^\star_t}-\max_{k\ne k^\star_t}\ell^{(t)}_k$ is
positive, and for each conditional branch with flag logit $z_t$, the margin
$m_t = |z_t|$ is positive. Let $\Delta_{\min}=\min_t \Delta_t>0$ and
$m_{\min}=\min_t m_t>0$.
\end{assumption}

In the executed dispatch the flags are exact bits, so $z_t = \pm 1$ and
$m_{\min}=1$, so Assumption~\ref{ass:dm} holds trivially for branches and is a
statement only about ties in a fully soft dispatch.

\begin{theorem}[Temperature-limit bound]
\label{thm:limit}
Consider the fully relaxed emulator that uses the softmax
dispatch~\eqref{eq:select}, sigmoid branch~\eqref{eq:branch}, and a
distance-softmax (temperature-$T$) read \emph{without} the straight-through
correction. Under Assumption~\ref{ass:dm}, its one-step map converges to the hard
map as $T\to0$ and $\alpha\to\infty$, with per-step deviation
\begin{equation}
\begin{aligned}
\big\|\Phi^{T}_{\mathrm{S}}(x)-\Phi_{\mathrm{H}}(x)\big\|_\infty
\;\le\; C\Big[&(K-1)\,e^{-\Delta_{\min}/T}\\
&{}+ e^{-\alpha m_{\min}} + \rho\,e^{-1/T}\Big],
\end{aligned}
\label{eq:bound}
\end{equation}
where $C$ bounds the value range (bytes $\le255$, branch offsets $\le256$) and
$\rho$ the number of reads per step. Over $N$ steps, with Lipschitz constant
$L\ge1$ for error propagation, the trajectory deviation is at most
$\frac{L^{N}-1}{L-1}$ times~\eqref{eq:bound}, which is $\mathcal{O}(e^{-c/T})$
with $c=\min(\Delta_{\min},1)$.
\end{theorem}

\begin{proof}[Proof sketch]
Bound the non-winning softmax mass by $(K-1)e^{-\Delta_{\min}/T}$, the sigmoid
error by $e^{-\alpha m_{\min}}$, and the temperature-read pull by $e^{-1/T}$ (no
margin needed), then propagate over $N$ steps with the Lipschitz constant $L$. The
complete proof is in the supplementary material.
\end{proof}

\begin{corollary}
\label{cor:ste}
The straight-through estimator~\eqref{eq:ste} sets the bracket
in~\eqref{eq:bound} to zero for all $T,\alpha$: its forward value is the hard
($T\to0$) limit \emph{exactly at finite temperature}, while its backward pass
uses the chosen soft relaxation's gradient. It is therefore a
\emph{surrogate-gradient estimator}---the forward equals the hard machine, the
gradient does not differentiate it---not the limiting hard derivative.
Theorem~\ref{thm:exact} is thus the $T\to0$ forward limit, achieved by
construction.
\end{corollary}

These results are matched by the implementation's unit tests, which confirm
that the softmax dispatch saturates to the exact hard pick at large logits,
that the sigmoid branch saturates to the exact taken/not-taken program
counter at large $\alpha$, and that the soft memory read collapses to
ordinary indexing at small $T$---the empirical face of
Theorems~\ref{thm:exact}--\ref{thm:limit}.

\section{Evaluation Methodology}

We report two conformance metrics and an implementation-effort estimate, defined
here so the results are unambiguous. The two conformance metrics are evaluated on the 64 ALE games in
our \textsc{xitari} configuration---\textsc{xitari} being the emulator DeepMind
released with DQN; the canonical DQN results were reported on a 49-game subset
\citep{mnih2015human}, later work on an Atari-57 set, all drawn from ALE's larger
catalogue \citep{machado2018revisiting}. For each game we run the standard ALE
boot (60 no-op frames plus four resets) and then a fixed action stream. A game is
\emph{RAM-exact} if its 128-byte RIOT RAM is identical to \textsc{xitari} on every
frame, and \emph{pixel-exact} if its full $210\times160$ frame buffer (palette
indices) is identical on every frame. We report the count of each out of 64. We
estimate \emph{active} implementation wall-time from the version-control log.
\textsc{jutari} is the faster engine per single environment on the CPU, whereas
\textsc{jaxtari} scales on the GPU---batched differentiable rollouts reach
${\sim}3$M environment-steps per second on a commodity card---with the full
single-step, compiled-rollout, batched, and GPU numbers in the supplementary material.

\section{Results}

Both ports are bit-for-bit identical to \textsc{xitari} on \emph{all} 64 ALE
games: 64/64 games are RAM-exact (zero differing bytes across every frame) and
64/64 are pixel-exact (zero differing pixels across every frame).

\subsection{Effort: Person-Months into Days}

The project comprises 373 commits over a 29-day calendar span (a cumulative
commit-time plot is in the supplement). Splitting the commit stream at three-hour
gaps---the session boundary defined in the methodology, with commits within a
session a median ${\sim}13$ minutes apart---yields 52 active sessions totalling
\textbf{136.8 hours}, about \textbf{5.7 round-the-clock days}, and the split is
insensitive to the cutoff (106 hours at two hours, 162 at four). The remaining
${\sim}80\%$ of the span was idle, mostly travel without reliable internet. The
work used Anthropic's Claude Opus~4.7 and 4.8 as the primary coding models, with a
brief autonomous stint by Fable~5. Against this we built two independent emulators
totalling ${\sim}18{,}700$ lines of source and a comparable volume of tests
(${\sim}1{,}990$ cases across the two ports), and chased 71 numbered bug
investigations to a bit-exact result.

For a baseline, \textsc{xitari} is a C\texttt{++} core of about 46{,}000 lines
derived from Stella, with sources dating to the 1990s and an \textsc{xitari} fork
spanning 2014--2017 \citep{xitari,stella}---many person-years of work. Our two
ports reproduce its behaviour (bit-exact RAM, pixel-exact screens) in about
18{,}700 lines and 137 active hours. We do not claim a controlled comparison---the
reference is larger in scope (audio, every cartridge type, tooling)---but the
contrast between a multi-year reference and a six-day reimplementation is the right
order of magnitude for what agentic tools change.

\subsection{Proof of Concept: Ground-Truth Gradients}

\begin{figure}[t]
\centering
\includegraphics[width=\columnwidth]{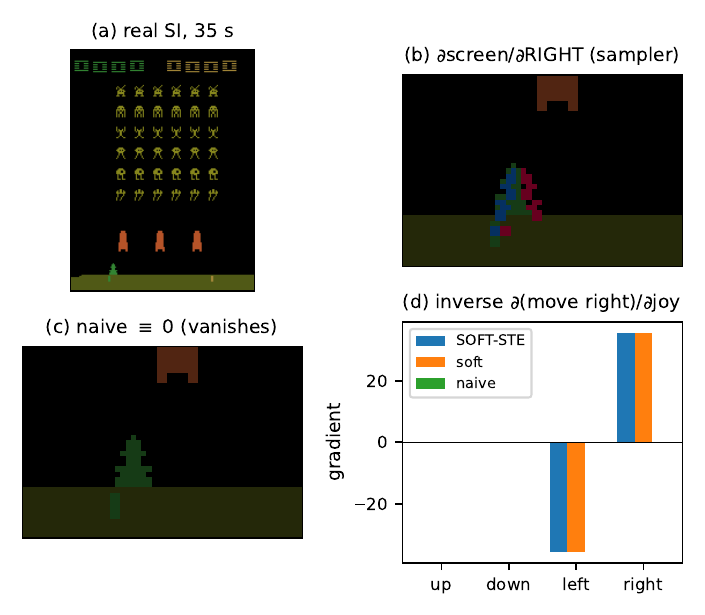}
\caption{Ground-truth gradients on the \emph{real} Space Invaders ROM in the
differentiable VCS. \textbf{(a)} the live game $35$\,s after boot---the scene the
real ROM renders, pixel-exact to \textsc{xitari}. \textbf{(b)} the screen's
directional derivative with respect to the \textsc{right} joystick, recovered by a
differentiable sampler over the player-cannon position (a sub-pixel bilinear
kernel, as in spatial-transformer networks); it lights exactly the cannon edges
that move. Because soft mode is forward-exact (Theorem~\ref{thm:exact}), this
gradient is \emph{identical} for SOFT-STE and the soft (relaxed) variant (the
leaky, past-boundary variant is studied in the supplement). \textbf{(c)} the
\emph{naive} gradient runs through the discrete sprite-position index and is
identically zero---it would wrongly report ``the joystick does not affect the
screen.'' \textbf{(d)} the inverse problem,
$\partial(\text{move cannon right})/\partial\text{joystick}$: SOFT-STE and soft
recover the control mapping (push \textsc{right}, not up/down), while the naive
gradient vanishes. Every recovered quantity can be \emph{checked} against the
known wiring---the whole point of a ground-truth substrate.}
\label{fig:xai}
\end{figure}

The point of the foundation is that gradients computed on it can be
\emph{checked}. Because soft mode is forward-exact (Theorem~\ref{thm:exact}), the
surrogate gradients are taken along the true hard trajectory, not an
approximation; we illustrate this end-to-end on the \emph{real} Space Invaders
ROM, rolled to a live scene (Figure~\ref{fig:xai}a). It is a proof of concept
that the substrate behaves as intended, not an XAI evaluation.

The joystick moves the cannon, but the path to the screen runs through a discrete
sprite \emph{position} index, so a naive gradient is identically zero
(Figure~\ref{fig:xai}c)---the ground truth itself warning that a naive saliency
would report ``the joystick does not affect the screen.'' A differentiable
sampler over the position (a sub-pixel bilinear kernel, as in spatial-transformer
networks~\citep{jaderberg2015spatial}) restores it: the per-direction map lights
exactly the cannon edges that move (Figure~\ref{fig:xai}b), and the inverse
recovers the control mapping---push \textsc{right}, not up/down
(Figure~\ref{fig:xai}d), identically for SOFT-STE and the soft relaxed variant
(the leaky variant is in the supplement). Because the wiring is known these
gradients can be \emph{scored}: the sampler puts $100\%$ of its
$\partial$screen$/\partial\textsc{right}$ mass on the cannon and the naive map
none ($\equiv0$), and the inverse Jacobian is $+35.7/{-}35.7$ for
\textsc{right}/\textsc{left} and $0$ up/down (axis and sign right), the naive
gradient silent.

\section{Discussion}

This paper delivers five results. We built two independent differentiable ports of
the Atari VCS, in Julia and JAX, bit-exact in RAM and pixel-exact on screen against
\textsc{xitari} on all 64 DQN games; proved the soft forward pass equals the hard
one exactly, so gradients run on the true execution rather than an approximation;
measured the implementation effort at about 137 active hours, built largely by
coding agents; found that a bit-exact re-implementation is itself an audit tool for
its reference; and showed qualitatively that gradients localise to the correct
hardware and can be checked against the known mechanism. A theoretical analysis of
$\alpha$ and $T$ shows that the fully relaxed variant, too, is forward-exact for
small $T$ and large $\alpha$. The supplement collects the full proofs, the
numerical relaxation analyses ($\alpha$/$T$ gradient maps and the
forward-bit-exactness boundaries), and the beam-timed comparison videos.

Building the ports also revealed a bug in the reference. Driving them to 64/64
pixel exactness surfaced a latent non-determinism in \textsc{xitari} itself: one
title's attract-mode demo reads \emph{uninitialised} on-cartridge Superchip RAM as
a random-number source, and \textsc{xitari} seeds that RAM from
\texttt{time(NULL)} while ignoring its own seed parameter, so the title renders
differently on every run despite a pinned seed---a divergence invisible to RAM
checks and visible only on screen. We reach full conformance by pinning the seed
deterministically in both emulators and have prepared an upstream fix: a bit-exact
re-implementation is useful in both directions, the reference validating the port
and the port auditing the reference.

The hardest divergences were all in timing---the RIOT timer's reset state, the
TIA's sub-cycle object rendering, and the boot-time format probe that seeds
free-running counters games never re-initialise. As one model anecdote, an early
attempt with Kimi-K2.6 (a one-trillion-parameter model) ended when it refused to
start, judging the port to be months of work.

\paragraph{What we do not claim.} This paper builds and validates the substrate;
it is not yet an XAI benchmark. We neither rank attribution methods nor study a
learned agent---the gradients here explain the \emph{known machine}, not a policy,
and are surrogate estimators of the hard map (Corollary~\ref{cor:ste}), not the
hard derivative.

\paragraph{Limitations.} The effort figure is an order-of-magnitude estimate, not
a controlled trial; this does not affect the core claims (Theorem~\ref{thm:exact}
and the measured 64/64 conformance). Audio is not implemented, being irrelevant to
a pixel- and RAM-level XAI substrate. Evaluating attribution methods on this
ground truth, and extending it from the known environment to learned agents, is
the program this foundation opens.

\section{Conclusion}

We set out to give explainable AI a study object that escapes the
known-but-trivial / complex-but-unknown dichotomy: a genuine, complex
information-processing system that is at once fully specified and fully
differentiable. We built it twice and validated both ports bit-for-bit against the
reference, with a forward pass proven bit-exact yet exposing \emph{surrogate}
gradients where bit logic has none. An attribution can now---to our knowledge, for
the first time on a system of this complexity---be scored against the truth: the
XAI testbed we hope the community takes up.

\section*{Ethical Statement}
This work involves no human subjects, personal data, or sensitive information. The
commercial Atari ROMs are used only as fixed inputs; they are copyrighted and
\emph{not} redistributed---we release SHA-256 hashes for identification only. The
agentic method is dual-use; pairing autonomy with a strict, automatically checked
conformance oracle mitigates its central risk, silent divergence from the
reference, without eliminating it.

\paragraph{Supplementary video.} A narrated walkthrough of the two ports, the
bit-exact side-by-side comparisons, and the gradient study is available at
\url{https://github.com/akmaier/UnderstandingVCS/blob/main/jutari_paper/presentation/presentation.mp4}.

{\small
\bibliography{references}
}

\clearpage
\appendix
\section*{Supplementary Material}

This supplement collects the analyses behind the main paper. We give the full
proofs of the two theorems; a numerical study of how the relaxation parameters
$\alpha$ and $T$ shape the gradient; an empirical and analytic account of when the
\emph{fully relaxed} forward stays bit-exact (a cast-margin model, its
$(\alpha,T)$ likelihood map, and boundary measurements on the full simulator),
accompanied by the beam-timed comparison videos; and the implementation-effort
timeline referenced in the Results. It is self-contained: we first restate the
soft/hard formulation, prove the forward-equivalence and temperature-limit
results, then study the relaxation gradient and forward bit-exactness, and finally
show the effort timeline.

\section{Setup and Notation}

The differentiable emulator runs one of two transition maps over the full VCS
state $x=(\text{registers},\text{RAM},\text{RIOT},\text{TIA},\text{frame
buffer})$: the bit-exact \emph{hard} map $\Phi_{\mathrm{H}}$ and the
differentiable \emph{soft} map $\Phi_{\mathrm{S}}$. In both, the handler that
runs is the one for the opcode byte at the program counter $\mathrm{pc}$. The
handlers are assembled from five primitives, which we restate here so that the
proofs are self-contained.

The first primitive is the memory read. A read of a memory tensor
$r\in\mathbb{R}^{M}$ of size $M$ (the ROM or the RAM) at an integer address
$a\in\{0,\dots,M-1\}$ is written as a dot product against the one-hot address
indicator $\mathbf{1}_a$,
\begin{equation}
\mathrm{peek}(r,a)=\mathbf{1}_a^{\top}r=\textstyle\sum_i[\![i=a]\!]\,r_i=r_a ,
\label{eq:s-peek}
\end{equation}
so the forward value is the array element $r_a$, while the gradient with respect
to $r$ is the one-hot vector $\mathbf{1}_a$.

The second primitive is opcode dispatch. With a score vector (logits)
$\ell\in\mathbb{R}^{K}$ over the $K=256$ opcodes, a temperature $T>0$, and the
softmax weights $w=\softmax(\ell/T)$, the relaxed dispatch over the candidate
handler-output rows $V$ is the convex combination
\begin{equation}
\mathrm{select}(\ell,V;T)=w^{\top}V,\qquad
w_k=\frac{e^{\ell_k/T}}{\sum_j e^{\ell_j/T}} .
\label{eq:s-select}
\end{equation}

The third primitive is the conditional branch. For a relaxed status flag
$f\in[0,1]$ written as a logit $z=s(2f-1)$, where the sign $s$ selects
branch-when-set or branch-when-clear, and a sharpness $\alpha$, the gate and the
relaxed program counter are
\begin{equation}
g=\sigma(\alpha z),\qquad
\mathrm{pc}_{\mathrm{soft}}=(1-g)\,\mathrm{pc}_{\bar b}+g\,\mathrm{pc}_{b}
=\mathrm{pc}_{\bar b}+g\,\delta ,
\label{eq:s-branch}
\end{equation}
where $\mathrm{pc}_{\bar b}$ is the fall-through (not-taken) counter and
$\mathrm{pc}_{b}=\mathrm{pc}_{\bar b}+\delta$ is the taken counter for the signed
branch offset $\delta$.

The fourth primitive is the straight-through estimator. For any soft value that
has a matching hard value it returns
\begin{equation}
\mathrm{STE}(\text{soft},\text{hard})=\text{soft}+\sg(\text{hard}-\text{soft}),
\label{eq:s-ste}
\end{equation}
where $\sg(\cdot)$ is the stop-gradient operator, so the forward value equals
\textit{hard} while the backward pass uses the gradient of \textit{soft}.

The fifth primitive is the \emph{relaxed read}, the differentiable counterpart of
the exact read~\eqref{eq:s-peek}. It blurs the integer address with a
distance-softmax of width $T$, returning the temperature-weighted average of the
neighbouring cells,
\begin{equation}
\mathrm{peek}_{\mathrm{R}}(r,a;T)=\textstyle\sum_k w_k\,r_{a+k},\qquad
w_k=\frac{e^{-|k|/T}}{\sum_j e^{-|j|/T}},
\label{eq:s-read}
\end{equation}
so that $\mathrm{peek}_{\mathrm{R}}\!\to r_a$ as $T\to0$ while the address carries
a nonzero gradient for $T>0$. The executed (SOFT-STE) path uses the exact one-hot
read~\eqref{eq:s-peek}; only the fully relaxed variant of
Theorem~\ref{thm:s-limit} uses~\eqref{eq:s-read}.

In the executed soft step, dispatch uses the saturated (hard) form of
\eqref{eq:s-select} and the branch returns
$\mathrm{STE}(\mathrm{pc}_{\mathrm{soft}},\mathrm{pc}_{\mathrm{hard}})$. The
relaxed forms \eqref{eq:s-select} and \eqref{eq:s-branch} \emph{without} the
straight-through correction define the ``fully relaxed'' variant analysed in
Theorem~\ref{thm:s-limit}. The temperature-limit result needs one
non-degeneracy assumption, identical to the one in the main paper.

\begin{assumption}[Decision margins]
\label{ass:s-dm}
Along the executed trajectory of length $N$, every discrete decision has a
strictly positive margin. For opcode dispatch at step $t$ with logits
$\ell^{(t)}$ and winner $k^\star_t$, the logit gap
$\Delta_t=\ell^{(t)}_{k^\star_t}-\max_{k\ne k^\star_t}\ell^{(t)}_k$ is positive,
and for each conditional branch with flag logit $z_t$, the margin $m_t=|z_t|$ is
positive. Let $\Delta_{\min}=\min_t\Delta_t>0$ and $m_{\min}=\min_t m_t>0$.
\end{assumption}

These positive margins are a property of the \emph{executed} length-$N$
trajectory; geometrically they define a tube around it inside which the relaxed
maps agree with the hard one. The relaxed read~\eqref{eq:s-read} needs no margin:
its off-target weights decay unconditionally with $T$ (proved below), so only
dispatch and branches require Assumption~\ref{ass:s-dm}.

The three execution modes referenced throughout are summarised in
Table~\ref{tab:s-modes}.

\begin{table}[tbh]
\centering
\caption{The three execution modes. \textsc{soft-ste} is the mode used for every
gradient in this paper: its forward pass is bit-identical to \textsc{hard}
(Theorem~\ref{thm:s-exact}) and its backward pass is a surrogate. The fully
relaxed mode (\textsc{full}) is used only for the temperature-limit analysis
(Theorem~\ref{thm:s-limit}); its forward pass is bit-exact only inside the corner
of small $T$ and large $\alpha$.}
\label{tab:s-modes}
\small
\setlength{\tabcolsep}{4pt}
\begin{tabular}{@{}llll@{}}
\toprule
Mode & Forward & Gradient & Used for\\
\midrule
\textsc{hard} & bit-exact & none & conformance\\
\textsc{soft-ste} & $=$\,\textsc{hard} & surrogate & attribution\\
\textsc{full} & exact in corner & relaxed & $T{\to}0$ study\\
\bottomrule
\end{tabular}
\end{table}

\paragraph{Numerical scope.} Soft mode keeps every quantity in float32 but only
ever at integer values in $[-2^{24},2^{24}]$ (Theorem~\ref{thm:s-exact}), so no
representation error accrues and the equalities below are exact, not approximate.
Both ports apply the \emph{same} round-to-integer after each soft primitive, so
\textsc{jutari} (Zygote) and \textsc{jaxtari} (JAX/XLA) produce identical integer
trajectories. The float path carries gradients only; a quantity outside the
$24$-bit range---which the VCS never produces---would forfeit the exactness
guarantee.

\section{Proofs}

\begin{theorem}[Exact forward equivalence]
\label{thm:s-exact}
For every reachable state $x$ and every finite sharpness $\alpha$ and
temperature $T$, the soft and hard one-step maps are equal by design,
$\Phi_{\mathrm{S}}(x) = \Phi_{\mathrm{H}}(x)$, with identical float32
representations. Consequently, over a trajectory of $N$ steps,
$x^{\mathrm{S}}_t = x^{\mathrm{H}}_t$ for all $t\le N$. The modes differ only in
the Jacobian $\partial\Phi_{\mathrm{S}}/\partial x$, which is defined and
generically nonzero where $\partial\Phi_{\mathrm{H}}/\partial x$ is zero or
undefined.
\end{theorem}

\begin{proof}
We first prove the one-step equality $\Phi_{\mathrm{S}}(x)=\Phi_{\mathrm{H}}(x)$
for an arbitrary reachable $x$, and then extend it to a trajectory of length $N$
by induction on the step index.

\textit{One-step equality.} Fix a reachable $x$. In both modes the handler that
executes is selected by the decoded opcode byte through the same hard switch (the
softmax form~\eqref{eq:s-select} is not evaluated on the forward path), so both
modes run the same handler on the same input. It therefore suffices to show that
each primitive the handler uses returns the same forward value in both modes. A
memory read at the integer address $a$ uses an exact one-hot indicator, so
\begin{equation}
\mathrm{peek}_{\mathrm{S}}(r,a)=\mathbf{1}_a^{\top}r
=\textstyle\sum_i[\![i=a]\!]\,r_i=r_a=\mathrm{peek}_{\mathrm{H}}(r,a),
\label{eq:s-peek-eq}
\end{equation}
the value a hard array read returns. Register, status-flag,
binary-coded-decimal, and timer updates use the same integer and $\mathrm{round}$
arithmetic in both modes (soft mode merely keeps the results in float32). A
conditional branch returns the straight-through value~\eqref{eq:s-ste}. Since
$\sg$ is the identity in the forward direction,
\begin{equation}
\begin{aligned}
\mathrm{STE}(\mathrm{pc}_{\mathrm{soft}},\mathrm{pc}_{\mathrm{hard}})
&=\mathrm{pc}_{\mathrm{soft}}+\sg(\mathrm{pc}_{\mathrm{hard}}-\mathrm{pc}_{\mathrm{soft}})\\
&=\mathrm{pc}_{\mathrm{hard}}
\end{aligned}
\label{eq:s-branch-eq}
\end{equation}
for every $\alpha$ and $T$. Every component of the handler output thus equals its
hard counterpart, and as the two handlers act on the same input,
\begin{equation}
\Phi_{\mathrm{S}}(x)=\Phi_{\mathrm{H}}(x).
\label{eq:s-onestep}
\end{equation}
These equalities hold in float32, not merely over the reals: IEEE-754 single
precision has a $24$-bit significand and so represents every integer in
$[-2^{24},2^{24}]$ exactly~\citep{ieee754_2019,goldberg1991}, and every VCS
quantity lies in this range (data bytes $\le255$, $13$-bit addresses, per-frame
cycle counts of a few tens of thousands), so
\eqref{eq:s-peek-eq}--\eqref{eq:s-onestep} are exact.

\textit{Induction on the trajectory.} Let the two runs share the initial state
and evolve by $x^{\bullet}_{t+1}=\Phi_{\bullet}(x^{\bullet}_t)$. We show
$x^{\mathrm{S}}_t=x^{\mathrm{H}}_t$ for all $0\le t\le N$.

\textit{Base case} ($t=0$): $x^{\mathrm{S}}_0=x^{\mathrm{H}}_0$ by the shared
initial state.

\textit{Inductive step:} assume $x^{\mathrm{S}}_t=x^{\mathrm{H}}_t=:x$ for some
$t<N$. Applying the one-step equality~\eqref{eq:s-onestep} to $x$,
\begin{equation}
x^{\mathrm{S}}_{t+1}=\Phi_{\mathrm{S}}(x^{\mathrm{S}}_t)=\Phi_{\mathrm{S}}(x)
=\Phi_{\mathrm{H}}(x)=\Phi_{\mathrm{H}}(x^{\mathrm{H}}_t)=x^{\mathrm{H}}_{t+1},
\end{equation}
where the outer equalities use the inductive hypothesis and the middle one
is~\eqref{eq:s-onestep}. By induction $x^{\mathrm{S}}_t=x^{\mathrm{H}}_t$ for all
$t\le N$.

\textit{Gradients.} The stop-gradient in~\eqref{eq:s-ste} alters only the
reverse-mode rule: the returned counter carries the derivative of
$\mathrm{pc}_{\mathrm{soft}}=\mathrm{pc}_{\bar b}+g\,\delta$, namely
$\alpha\,g(1-g)\,\delta$, which is generically nonzero where the hard branch---a
step function of the flag---has zero or undefined derivative. Likewise the
read~\eqref{eq:s-peek} has Jacobian $\mathbf{1}_a$ with respect to $r$. The
forward pass is unchanged while gradients become available.
\end{proof}

\begin{theorem}[Temperature-limit bound]
\label{thm:s-limit}
Consider the fully relaxed emulator that uses the softmax
dispatch~\eqref{eq:s-select}, sigmoid branch~\eqref{eq:s-branch}, and
distance-softmax reads~\eqref{eq:s-read} \emph{without} the straight-through
correction. Under Assumption~\ref{ass:s-dm}, its one-step map converges to the
hard map as $T\to0$ and $\alpha\to\infty$, with per-step deviation
\begin{equation}
\begin{aligned}
\big\|\Phi^{T}_{\mathrm{S}}(x)-\Phi_{\mathrm{H}}(x)\big\|_\infty
\;\le\; C\big[&(K-1)\,e^{-\Delta_{\min}/T}\\
&{}+e^{-\alpha m_{\min}}+\rho\,e^{-1/T}\big],
\end{aligned}
\label{eq:s-bound}
\end{equation}
where $C$ bounds the value range (bytes $\le255$, branch offsets $\le256$) and
$\rho$ the number of reads in the step. Over $N$ steps, with Lipschitz constant
$L\ge1$ for error propagation, the trajectory deviation is at most
$\frac{L^{N}-1}{L-1}$ times the right-hand side of~\eqref{eq:s-bound}, which is
$\mathcal{O}(e^{-c/T})$ with $c=\min(\Delta_{\min},1)$.
\end{theorem}

\begin{proof}
We bound the deviation of the relaxed one-step map from the hard one and then
propagate it along the trajectory. Throughout, $C$ bounds the range of any state
value (bytes $\le255$, branch offsets $\le256$).

\textit{Dispatch term.} At a step with logits $\ell$ and winner $k^\star$ the
gap is $\Delta\ge\Delta_{\min}>0$ by Assumption~\ref{ass:s-dm}, so for any loser
$k\ne k^\star$,
\begin{equation}
w_k=\frac{e^{\ell_k/T}}{\sum_j e^{\ell_j/T}}
\le\frac{e^{\ell_k/T}}{e^{\ell_{k^\star}/T}}
=e^{(\ell_k-\ell_{k^\star})/T}\le e^{-\Delta_{\min}/T},
\end{equation}
so the non-winning mass is $1-w_{k^\star}=\sum_{k\ne k^\star}w_k\le
(K-1)e^{-\Delta_{\min}/T}$. As the dispatch output is the convex combination
$\sum_k w_k V_k$,
\begin{equation}
\begin{aligned}
\Big\|\textstyle\sum_k w_k V_k-V_{k^\star}\Big\|_\infty
&\le(1-w_{k^\star})\max_k\|V_k-V_{k^\star}\|_\infty\\
&\le(K-1)\,e^{-\Delta_{\min}/T}\,C .
\end{aligned}
\label{eq:s-disp}
\end{equation}

\textit{Branch term.} For a margin $m=|z|\ge m_{\min}$ the gate error is
\begin{equation}
\big|\sigma(\alpha z)-[\![z>0]\!]\big|=\sigma(-\alpha m)\le e^{-\alpha m_{\min}},
\label{eq:s-gate}
\end{equation}
so the blended counter $\mathrm{pc}_{\bar b}+g\,\delta$ differs from the hard
target $\mathrm{pc}_{\bar b}+[\![z>0]\!]\,\delta$ by
\begin{equation}
\big|(g-[\![z>0]\!])\,\delta\big|\le|\delta|\,e^{-\alpha m_{\min}}
\le C\,e^{-\alpha m_{\min}}.
\label{eq:s-branchterm}
\end{equation}

\textit{Read term.} A distance-softmax read~\eqref{eq:s-read} deviates from the
hard value $r_a$ by the neighbour pull. Using $|r_{a+k}-r_a|\le 2C$ and
$1-w_0=\big(\sum_{k\ne0}e^{-|k|/T}\big)/\big(\sum_j e^{-|j|/T}\big)
\le 2\sum_{k\ge1}e^{-k/T}=2e^{-1/T}/(1-e^{-1/T})$,
\begin{equation}
\Big|\textstyle\sum_{k\ne0}w_k(r_{a+k}-r_a)\Big|
\le (1-w_0)\,2C\le C\,e^{-1/T}
\label{eq:s-readterm}
\end{equation}
for $T$ below a fixed $T_0$, absorbing the geometric factor and the byte range
into $C$. Unlike dispatch and branch this requires no margin: the integer address
spacing is $\ge1$, so the off-target weights decay as $e^{-1/T}$ unconditionally.
The $\rho$ reads of the step contribute at most $\rho\,C\,e^{-1/T}$.

\textit{Per-step bound.} Adding~\eqref{eq:s-disp}, \eqref{eq:s-branchterm},
and~\eqref{eq:s-readterm} bounds the one-step deviation: for every reachable $x$,
\begin{equation}
\begin{aligned}
\delta_{\mathrm{step}}
&:=\big\|\Phi^{T}_{\mathrm{S}}(x)-\Phi_{\mathrm{H}}(x)\big\|_\infty\\
&\le C\big[(K-1)e^{-\Delta_{\min}/T}+e^{-\alpha m_{\min}}+\rho\,e^{-1/T}\big],
\end{aligned}
\label{eq:s-perstep}
\end{equation}
which is~\eqref{eq:s-bound}.

\textit{Propagation.} Let $L\ge1$ be a Lipschitz constant of $\Phi_{\mathrm{H}}$
in $\|\cdot\|_\infty$ and $e_t=\|x^{T,\mathrm{S}}_t-x^{\mathrm{H}}_t\|_\infty$ the
trajectory error after $t$ steps, with $e_0=0$. Each step adds at most
$\delta_{\mathrm{step}}$ to the hard map applied to the accumulated error,
\begin{equation}
e_{t+1}\le L\,e_t+\delta_{\mathrm{step}},
\end{equation}
and unrolling from $e_0=0$ gives
\begin{equation}
e_N\le\delta_{\mathrm{step}}\sum_{i=0}^{N-1}L^i
=\delta_{\mathrm{step}}\,\frac{L^N-1}{L-1}.
\end{equation}
As $T\to0$ and $\alpha\to\infty$,
$\delta_{\mathrm{step}}=\mathcal{O}(e^{-c/T})\to0$ with
$c=\min(\Delta_{\min},1)$, so the relaxed trajectory converges to the hard one,
$\Phi^{T}_{\mathrm{S}}\to\Phi_{\mathrm{H}}$.
\end{proof}

\section{Effect of the Relaxation Parameters on the Gradient}

Theorems~\ref{thm:s-exact} and~\ref{thm:s-limit} say the hard limit
($\alpha\to\infty$ for the sigmoid branch, $T\to0$ for the softmax select) is
exact in value but degenerate in gradient. We illustrate both knobs on \textsc{jutari}
with a single target sprite (an invader) whose horizontal placement is produced
by the soft primitives. For the sigmoid branch we plot the screen-space
directional-derivative saliency
$\partial(\text{screen})/\partial(\text{control})$ as the sharpness $\alpha$ is
swept (Figure~\ref{fig:alphatemp}). The main paper additionally plots the gate
and its gradient directly, and shows the softmax select's temperature effect in
pixel space (a sampled sprite-occupancy heatmap). In the gradient maps the colour
encodes the
\emph{sign} of the derivative---red where a pixel brightens as the control
increases (the sprite arriving) and blue where it darkens (the sprite
leaving)---and its intensity the magnitude, on a black background. The colour bar
is in absolute units and the per-panel number is the peak relative to the row
maximum.

For the branch the position is $\mathrm{pc}=(1-g)\,\mathrm{pc}_L+g\,\mathrm{pc}_R$
with $g=\sigma(\alpha z)$, so its derivative with respect to the control $z$ is
$\alpha\,\sigma(\alpha z)\,(1-\sigma(\alpha z))\,(\mathrm{pc}_R-\mathrm{pc}_L)$:
a dipole, blue on the placement being vacated and red on the one being entered.
The magnitude $\alpha\,\sigma(\alpha z)(1-\sigma(\alpha z))$ is
\emph{non-monotonic} in $\alpha$ at the fixed operating point $z{=}0.25$. A small
$\alpha$ gives a shallow gate and hence a small slope, and a large $\alpha$
saturates the gate ($\sigma(\alpha z)\to1$, so $\sigma(1-\sigma)\to0$). The maximum
is the intermediate $\alpha$ for which $\alpha z$ falls in the sigmoid's steep band,
near $\alpha{=}6$ here---$\alpha{=}2,6,20$ give peak magnitudes
$0.53,1.00,0.15$, and at $\alpha{=}20$ the gate has essentially switched and the
gradient has nearly vanished (the hard-branch limit). The main paper plots this
gate and its gradient directly.

In every panel the forward render is the exact hard one
(Theorem~\ref{thm:s-exact}), and only the gradient changes. It is most informative
at intermediate relaxation and decays to zero as the relaxation is sharpened toward
the hard limit (Theorem~\ref{thm:s-limit}). This
is a qualitative study of how the parameter \emph{values} act on the gradient,
run on \textsc{jutari}, and the values need not be tuned for the bit-exact forward pass.

\begin{figure}[tbh]
\centering
\includegraphics[width=\columnwidth]{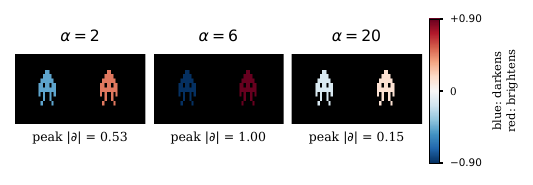}
\caption{Soft branch: the screen-space gradient
$\partial(\text{screen})/\partial z$ as the sharpness $\alpha$ is swept, on
\textsc{jutari}. The target sprite's placement is a sigmoid-gated blend of a left and a
right position, so the gradient is a dipole---blue where a pixel darkens (the
placement being vacated) and red where it brightens (the one being entered). The
colour bar is in absolute units (here $\pm0.90$) and the per-panel number is the
peak relative to the maximum. The dipole is strongest at an intermediate $\alpha$
(here $\alpha{=}6$) and nearly vanishes by $\alpha{=}20$ as the gate
saturates---the hard-branch limit (cf.\ Theorems~\ref{thm:s-exact}
and~\ref{thm:s-limit}).}
\label{fig:alphatemp}
\end{figure}

\section{Forward Bit-Exactness Under Full Relaxation}

The gradient study above keeps the straight-through estimator, so the forward
pass is bit-exact at any $\alpha$ and $T$ (Theorem~\ref{thm:s-exact}). To probe
the temperature-limit bound (Theorem~\ref{thm:s-limit}) empirically we instead
\emph{drop} the straight-through correction---the fully relaxed forward, where
$\alpha$ and $T$ act on the executed values themselves---and run it on the full
\textsc{jutari} simulator executing the real Space Invaders ROM. Every instruction casts
the sigmoid-blended program counter and the temperature-blended operand reads
back to integers, so the trajectory stays bit-exact only while every such cast
rounds to the value the executed (straight-through) path would produce: a large
$\alpha$ keeps the rounded soft program counter on the hard branch target, and a
small $T$ keeps the distance-softmax read on the addressed byte.

\begin{figure}[tb]
\centering
\includegraphics[width=\columnwidth]{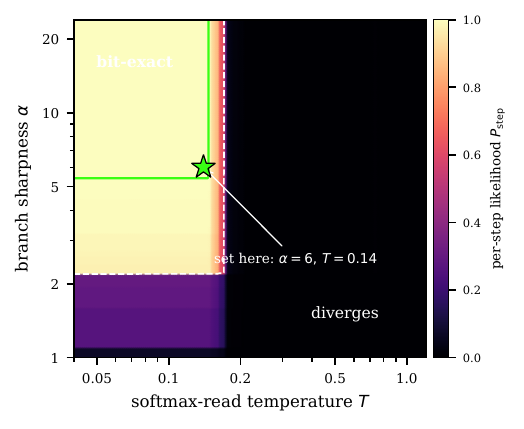}
\caption{Overview of the per-step bit-exactness likelihood
$P_{\mathrm{step}}=p_{\mathrm{read}}(T)^{\rho}\,p_{\mathrm{branch}}(\alpha)^{f_b}$
of the fully relaxed forward across the $(\alpha,T)$ plane (Space Invaders).
A bit-exact corner---bright, with the green contour marking
$P_{\mathrm{step}}{=}1$---sits at large branch sharpness $\alpha$ and small read
temperature $T$, bounded by a gradual branch boundary (in $\alpha$) and a sharp
temperature boundary (in $T$); the white dashed line is $P_{\mathrm{step}}{=}0.5$.
Because $P_{\mathrm{step}}$ factorises, the plane is the outer combination of the
two $1$D profiles, and the two boundary sweeps of Table~\ref{tab:relax-exact} are
slices along the green contour's two edges. The star marks the recommended
operating point ($\alpha{=}6$, $T{=}0.14$)---the corner where they cross.}
\label{fig:relaxheat}
\end{figure}

Because each cast is a rounding, the likelihood of staying bit-exact follows
from the cast margins. A crisp-flag branch with signed offset $\delta$ keeps its
rounded soft program counter on the hard target exactly when
$|\delta|<\tau_b(\alpha)$, $\tau_b(\alpha)=\tfrac12(1+e^{\alpha})$, so the
per-branch exactness $p_{\mathrm{branch}}(\alpha)$ is the fraction of executed
branch offsets below $\tau_b$. A temperature-$T$ read at address $a$ returns
$\mathrm{round}\!\big(\sum_k w_k\,\mathrm{mem}[a{+}k]\big)$ with
$w_k\propto e^{-|k|/T}$, and is exact when the neighbour pull
$\big|\sum_{k\neq0} w_k(\mathrm{mem}[a{+}k]-\mathrm{mem}[a])\big|<\tfrac12$,
giving a per-read exactness $p_{\mathrm{read}}(T)$. Treating the $\rho$ reads and
the occasional branch of an instruction as independent casts, the per-step
bit-exactness likelihood is
$P_{\mathrm{step}}(\alpha,T)=p_{\mathrm{read}}(T)^{\rho}\,
p_{\mathrm{branch}}(\alpha)^{f_b}$, with $\rho=2.16$ reads per instruction and
branch fraction $f_b=0.374$ measured on the executed Space Invaders trace, and
the run is bit-exact only while every cast rounds correctly.
$P_{\mathrm{step}}$ is a \emph{heuristic} likelihood---it treats the casts as
independent Bernoulli rounds, which they are not---so it predicts \emph{where}
exactness degrades rather than giving a probabilistic guarantee; the guarantee is
the deterministic worst-case cast condition, with $P_{\mathrm{step}}{=}1$ exactly
when every cast margin holds.

\begin{table}[tb]
\centering
\caption{Forward bit-exactness of the \emph{fully relaxed} pass (soft branch
without the straight-through estimator, temperature-$T$ reads) on the full
simulator running Space Invaders, sampled around the two critical boundaries.
\textbf{Measured} $N_{\mathrm{div}}$: the first instruction (of $N{=}3000$ from
reset) at which the relaxed trajectory deviates from the executed
(straight-through) one; ``exact'' means no deviation over the whole run.
\textbf{Predicted} (cast-margin model): per-branch exactness
$p_{\mathrm{branch}}(\alpha)$, per-read exactness $p_{\mathrm{read}}(T)$, and the
per-step likelihood
$P_{\mathrm{step}}=p_{\mathrm{read}}^{\rho}\,p_{\mathrm{branch}}^{f_b}$
($\rho=2.16$, $f_b=0.374$); the forward is bit-exact iff $P_{\mathrm{step}}=1$.
\textsc{jaxtari} uses the identical cast-to-Int logic and is omitted.}
\label{tab:relax-exact}
\small
\setlength{\tabcolsep}{4pt}
\begin{tabular}{rr c | ccc}
\toprule
\multicolumn{2}{c}{} & Meas. & \multicolumn{3}{c}{Predicted}\\
\cmidrule(lr){3-3}\cmidrule(lr){4-6}
$\alpha$ & $T$ & $N_{\mathrm{div}}$ & $p_{\mathrm{branch}}$ & $p_{\mathrm{read}}$ & $P_{\mathrm{step}}$\\
\midrule
\multicolumn{6}{l}{\emph{Branch boundary} ($T{=}0.14$, read-exact)}\\
$2$ & $0.14$ & $7$     & $0.039$ & $1.000$ & $0.298$\\
$3$ & $0.14$ & $879$   & $0.953$ & $1.000$ & $0.982$\\
$4$ & $0.14$ & $1121$  & $0.988$ & $1.000$ & $0.996$\\
$5$ & $0.14$ & $1216$  & $0.996$ & $1.000$ & $0.999$\\
$6$ & $0.14$ & \emph{exact} & $1.000$ & $1.000$ & $1.000$\\
\midrule
\multicolumn{6}{l}{\emph{Temperature boundary} ($\alpha{=}6$, branch-exact)}\\
$6$ & $0.08$ & \emph{exact} & $1.000$ & $1.000$ & $1.000$\\
$6$ & $0.10$ & \emph{exact} & $1.000$ & $1.000$ & $1.000$\\
$6$ & $0.12$ & \emph{exact} & $1.000$ & $1.000$ & $1.000$\\
$6$ & $0.14$ & \emph{exact} & $1.000$ & $1.000$ & $1.000$\\
$6$ & $0.15$ & $803$ & $1.000$ & $0.981$ & $0.959$\\
$6$ & $0.18$ & $2$   & $1.000$ & $0.223$ & $0.039$\\
$6$ & $0.20$ & $2$   & $1.000$ & $0.144$ & $0.015$\\
\bottomrule
\end{tabular}
\end{table}

Figure~\ref{fig:relaxheat} maps this likelihood across the whole $(\alpha,T)$
plane: a bit-exact corner at large $\alpha$ and small $T$, a gradual branch
boundary, and a sharp temperature boundary. Table~\ref{tab:relax-exact} zooms
into the two boundaries. As the faithful measurement
we report the first instruction at which the relaxed trajectory deviates from the
executed one over $N=3000$ steps (``exact'' means no deviation): the game
re-initialises much of its state every frame, so a diverged run can partially
re-synchronise and an end-of-run state match would overstate exactness.
Measurement and model agree---the forward is bit-exact precisely where
$P_{\mathrm{step}}=1$, namely $\alpha\ge6$ (so $\tau_b$ exceeds the largest
offset, $122$) and $T\le0.14$, and the first-divergence step grows as
$P_{\mathrm{step}}\to1$ ($\alpha{=}3,4,5$ at $T{=}0.14$ first diverge at steps
$879,1121,1216$). The branch boundary is gradual in $\alpha$ while the
temperature boundary is sharp: a single max-amplitude neighbour already pulls a
read by $\approx1.7>\tfrac12$, so $p_{\mathrm{read}}$ drops from $1$ at
$T\le0.14$ to $0.14$ at $T=0.20$. This is the empirical face of
Theorem~\ref{thm:s-limit}: exactness is recovered as $\alpha\to\infty$ and
$T\to0$. The \textsc{jaxtari} twin uses the identical cast-to-Int logic, so the same
analysis applies.

\begin{figure*}[p]
\centering
\includegraphics[width=\textwidth]{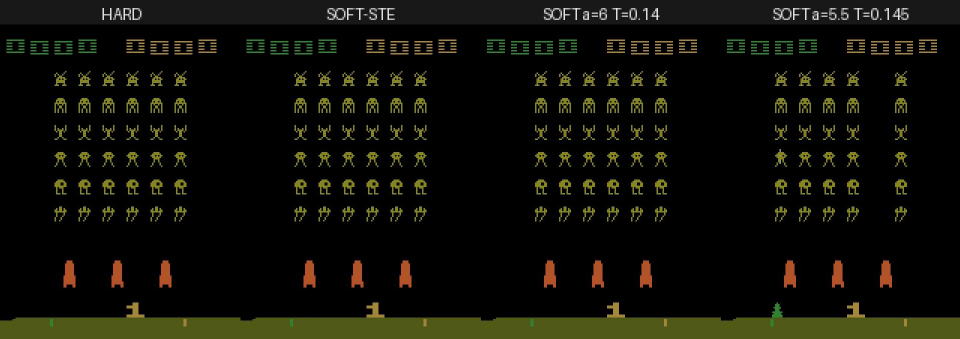}
\caption{One frame of the $60$-second full-simulator comparison (Space Invaders,
beam-timed). \textbf{HARD} and \textbf{SOFT-STE} are pixel-identical
(Theorem~\ref{thm:s-exact}). \textbf{SOFT} $\alpha{=}6,T{=}0.14$ is inside the
bit-exact corner and stays identical to HARD for the whole rollout. \textbf{SOFT}
$\alpha{=}5.5,T{=}0.145$ is just past the read boundary and diverges \emph{mildly}:
the game stays recognisable but a column of invaders is missing---the localised
effect of one corrupted sprite-position read, persistent over the rollout.}
\label{fig:divframe}
\end{figure*}

\begin{figure*}[p]
\centering
\includegraphics[width=\textwidth]{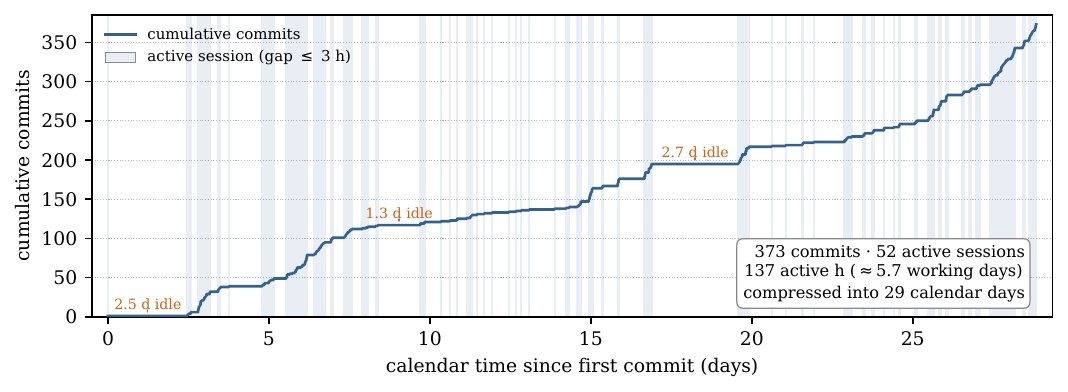}
\caption{Implementation effort reconstructed from the version-control log.
The step curve is the cumulative number of commits against calendar time, and
each shaded band is an active work session (a maximal run of consecutive
commits no more than three hours apart), and the long flat stretches between
bands are idle gaps, the largest of which are annotated. The $373$ commits form
$52$ active sessions totalling ${\sim}137$ hours of work---about $5.7$
round-the-clock days---spread across a $29$-day calendar window.}
\label{fig:timeline}
\end{figure*}

\begin{table*}[p]
\centering
\caption{Per-game conformance over the 64 ALE games supported by our \textsc{xitari} configuration. \textbf{Map}: mapper auto-detected by the ports (the set exercises seven of the ten implemented formats: 2K, 4K, F8, F6, E0, FE, F8SC; F4, F6SC and F4SC are implemented but unused here). \textbf{SHA-256}: first 12 hex digits of the ROM image (full hashes ship in the repository manifest; ROMs are not redistributed). \textbf{R}/\textbf{P}: RAM- and pixel-exact against \textsc{xitari} on every evaluated frame (30 NOOP frames for RAM, 60 frames of the fixed action stream for pixels, after the standard 60-NOOP+4-reset boot). All 64 games are exact in both (\checkmark); \textsc{jaxtari} matches \textsc{jutari} game-for-game.}
\label{tab:games}
\small
\setlength{\tabcolsep}{4pt}
\begin{tabular}{@{}lllcc@{}}
\toprule
Game & Map & SHA-256 & R & P\\
\midrule
air\_raid & 4K & \texttt{6b5a22ec074a} & \checkmark & \checkmark\\
alien & 4K & \texttt{9cd556c7d7f5} & \checkmark & \checkmark\\
amidar & 4K & \texttt{da5b760d3ada} & \checkmark & \checkmark\\
assault & 4K & \texttt{d32ca4dd9e84} & \checkmark & \checkmark\\
asterix & F8 & \texttt{b6161db0530d} & \checkmark & \checkmark\\
asteroids & F8 & \texttt{d1a6e808412b} & \checkmark & \checkmark\\
atlantis & 4K & \texttt{cef4b7b91ea2} & \checkmark & \checkmark\\
bank\_heist & 4K & \texttt{c32daffdb926} & \checkmark & \checkmark\\
battle\_zone & F8 & \texttt{efd6f639e724} & \checkmark & \checkmark\\
beam\_rider & F8 & \texttt{83b30cb52649} & \checkmark & \checkmark\\
berzerk & 4K & \texttt{cdff0422329d} & \checkmark & \checkmark\\
bowling & 2K & \texttt{dff44a85289f} & \checkmark & \checkmark\\
boxing & 2K & \texttt{54760d7e0710} & \checkmark & \checkmark\\
breakout & 2K & \texttt{376323f051c3} & \checkmark & \checkmark\\
carnival & 4K & \texttt{80eab759b6bd} & \checkmark & \checkmark\\
centipede & F8 & \texttt{6def669f54bd} & \checkmark & \checkmark\\
chopper\_command & 4K & \texttt{055637282252} & \checkmark & \checkmark\\
crazy\_climber & F8 & \texttt{486ad4cc8956} & \checkmark & \checkmark\\
defender & 4K & \texttt{e3c7ed7a073f} & \checkmark & \checkmark\\
demon\_attack & 4K & \texttt{8d72feeb267d} & \checkmark & \checkmark\\
double\_dunk & F6 & \texttt{c0ea53a91a39} & \checkmark & \checkmark\\
elevator\_action & F8SC & \texttt{f3b3ad03fd8c} & \checkmark & \checkmark\\
enduro & 4K & \texttt{ca9bb89755a6} & \checkmark & \checkmark\\
fishing\_derby & 2K & \texttt{84753bc6ecb4} & \checkmark & \checkmark\\
freeway & 2K & \texttt{3ef620234b98} & \checkmark & \checkmark\\
frostbite & 4K & \texttt{cbfdad89480d} & \checkmark & \checkmark\\
gopher & 4K & \texttt{b4aff03aeb0f} & \checkmark & \checkmark\\
gravitar & F8 & \texttt{17a19cd7b5ca} & \checkmark & \checkmark\\
hero & F8 & \texttt{75bd02ec7545} & \checkmark & \checkmark\\
ice\_hockey & 4K & \texttt{06df1fc568d9} & \checkmark & \checkmark\\
jamesbond & E0 & \texttt{995253f0d3a2} & \checkmark & \checkmark\\
journey\_escape & 4K & \texttt{963e2d3da52b} & \checkmark & \checkmark\\
\bottomrule
\end{tabular}\hfill\begin{tabular}{@{}lllcc@{}}
\toprule
Game & Map & SHA-256 & R & P\\
\midrule
kangaroo & F8 & \texttt{a2826eaea3a8} & \checkmark & \checkmark\\
krull & F8 & \texttt{b40f10d14217} & \checkmark & \checkmark\\
kung\_fu\_master & F8 & \texttt{3f6501a649ad} & \checkmark & \checkmark\\
montezuma\_revenge & E0 & \texttt{69d363a74754} & \checkmark & \checkmark\\
ms\_pacman & F8 & \texttt{dde0b43c5dee} & \checkmark & \checkmark\\
name\_this\_game & 4K & \texttt{cc219a7246e8} & \checkmark & \checkmark\\
pacman & 4K & \texttt{58e781b472e0} & \checkmark & \checkmark\\
phoenix & F8 & \texttt{b01085bc5227} & \checkmark & \checkmark\\
pitfall & 4K & \texttt{c719b47714d8} & \checkmark & \checkmark\\
pong & 2K & \texttt{41623e3c2614} & \checkmark & \checkmark\\
pooyan & 4K & \texttt{ac9abbfc272d} & \checkmark & \checkmark\\
private\_eye & F8 & \texttt{c8f9ce1b2e80} & \checkmark & \checkmark\\
qbert & 4K & \texttt{3257221832a7} & \checkmark & \checkmark\\
riverraid & 4K & \texttt{4c6842b8af64} & \checkmark & \checkmark\\
road\_runner & F6 & \texttt{4b6c2e68d693} & \checkmark & \checkmark\\
robotank & FE & \texttt{c186409481e6} & \checkmark & \checkmark\\
seaquest & 4K & \texttt{fbc29f4678f6} & \checkmark & \checkmark\\
skiing & 2K & \texttt{804eff5f0196} & \checkmark & \checkmark\\
solaris & F6 & \texttt{0afa36e5f4d8} & \checkmark & \checkmark\\
space\_invaders & 4K & \texttt{7224b17462b9} & \checkmark & \checkmark\\
star\_gunner & 4K & \texttt{3ad501a39280} & \checkmark & \checkmark\\
surround & 2K & \texttt{c6864f5c9f43} & \checkmark & \checkmark\\
tennis & 2K & \texttt{5a2052f020bf} & \checkmark & \checkmark\\
time\_pilot & F8 & \texttt{8ea97e335a2f} & \checkmark & \checkmark\\
tutankham & E0 & \texttt{403397fe8583} & \checkmark & \checkmark\\
up\_n\_down & F8 & \texttt{74890a43e591} & \checkmark & \checkmark\\
venture & 4K & \texttt{1870b0dd6fb1} & \checkmark & \checkmark\\
video\_pinball & 4K & \texttt{cef82c39bbb0} & \checkmark & \checkmark\\
videochess & 4K & \texttt{d44f2a115393} & \checkmark & \checkmark\\
wizard\_of\_wor & 4K & \texttt{d69ee89e65c3} & \checkmark & \checkmark\\
yars\_revenge & 4K & \texttt{ff777c8d4ea0} & \checkmark & \checkmark\\
zaxxon & F8 & \texttt{e1c323ce00cc} & \checkmark & \checkmark\\
\bottomrule
\end{tabular}
\end{table*}

\paragraph{Choosing $\alpha$ and $T$.} The two bounds are
\emph{program-independent}, which is what makes the setting reliable across
games. A $6502$ branch displacement is at most $|\delta|\le127$, so
$\tfrac12(1+e^{\alpha})>127$---that is $\alpha\ge\ln(2\cdot127-1)\approx5.5$---makes
every branch round correctly, and $\alpha\ge6$ suffices. A read is pulled away
from its addressed byte by at most $255\,(1-w_0)\approx510\,e^{-1/T}$, which stays
below $\tfrac12$ once $T\le 1/\ln(4\cdot255)\approx0.144$; the corner value
$T{=}0.14$ sits just inside this (worst-case pull $\approx0.40<\tfrac12$).
Both bounds depend only on the instruction set and the $8$-bit byte range, not on
the program, so $\alpha{=}6,\ T{=}0.14$ is bit-exact across games---verified over
$6000$ instructions on Pong, Breakout and Space Invaders---whereas \emph{just
outside} them the forward fails on a game-dependent subset ($\alpha{=}5$ diverges
on Space Invaders but not Pong or Breakout; $T{=}0.15$ diverges on Pong and Space
Invaders but not Breakout), which is exactly why the conservative bounds are
needed. Within the exact region the relaxed gradient is strongest at the boundary
(it vanishes as $\alpha\to\infty$, $T\to0$; Theorem~\ref{thm:s-limit}), so the
recommended operating point is the corner of the bit-exact region,
$\alpha{=}6,\ T{=}0.14$ (Figure~\ref{fig:relaxheat}): the smallest $\alpha$ and
largest $T$ that keep the forward bit-exact, giving the most informative gradient
compatible with exactness. We present this as a conservative, program-independent
operating point---its per-cast bounds are worst-case over the instruction set and
the $8$-bit byte range, and the cross-game exactness is confirmed empirically
(Pong, Breakout, Space Invaders)---not as a theorem that every program is exact at
this setting.

\paragraph{Full-simulator demonstration.} The per-cast model is borne out on a
real $60$-second rollout. We render Space Invaders beam-timed in four modes side
by side (Figure~\ref{fig:divframe}): the exact hard run, the executed soft
(straight-through) run, and the fully relaxed run at two settings. The relaxed
frames come from the same cycle-accurate hard renderer driven through a
default-off relaxation hook on the CPU's branch and reads, so HARD, SOFT-STE and
any in-corner setting render pixel-identically. At $\alpha{=}6,\ T{=}0.14$---inside
the bit-exact corner---the relaxed run stays identical to HARD for the entire
rollout. At $\alpha{=}5.5,\ T{=}0.145$---just past the read boundary---it diverges,
but only mildly: the game remains recognisable with a persistent, localised error
(a depleted column of invaders), because the relaxation corrupts a sprite-position
read rather than the control flow.

\paragraph{Delayed divergence.} A setting can also boot cleanly and diverge only
after many frames---the Bernoulli first failure deep into the rollout. The
mechanism is that the read margins are \emph{time-varying} for RAM: a data read
is exact while its neighbouring bytes are similar and crosses $\tfrac12$ only once
the game state evolves into a high-contrast layout, so the most-borderline cast
may be one the program reaches only late. Sweeping $T$ just below the read
boundary on a long clean-boot rollout (boot exact, relaxation enabled only for the
rollout) makes the first-divergence frame recede: at $\alpha{=}20$ the rollout
first deviates on frame~$1$ for $T{=}0.1448$, on frame~$194$ for $T{=}0.1446$, and
only on frame~$2365$ (${\approx}40$\,s of play) for $T\le0.1444$---reproducing the
game bit-for-bit for tens of seconds before a single late RAM read first rounds
wrong. (Below $T{=}0.1465$, the lowest fixed ROM-read threshold, no opcode/operand
fetch ever fails, so this late divergence is driven purely by an evolving RAM
read.) The exact--diverge boundary is thus a per-cast threshold \emph{spectrum}
over both fixed (ROM) and state-dependent (RAM) reads, not a single cliff.

\begin{table}[tb]
\centering
\caption{Throughput across engine $\times$ backend (real Pong ROM, soft mode,
$3{,}000$-step rollout, median of repeated runs after JIT warm-up). \textsc{jutari}
(Julia) executes one environment with inlined opcode handlers; \textsc{jaxtari} (JAX)
is slow per single environment but recovers---and on the GPU far exceeds---that by
\texttt{vmap}-batching the compiled rollout, its intended regime. Single-env rows are
steps/s ($=$ env-steps/s at batch~1); batched rows are aggregate env-steps/s at their
best batch ($N{\approx}4096$ on the GPU).}
\label{tab:s-compare}
\small
\setlength{\tabcolsep}{6pt}
\begin{tabular}{llr}
\toprule
engine \& backend & mode & env-steps/s\\
\midrule
\textsc{jutari}, CPU  & single env              & $370{,}100$\\
\textsc{jaxtari}, CPU & single env              & $1{,}178$\\
\textsc{jaxtari}, CPU & batched (\texttt{vmap})  & ${\sim}60{,}000$\\
\textsc{jaxtari}, GPU & batched (\texttt{vmap})  & $\mathbf{3{,}119{,}115}$\\
\bottomrule
\end{tabular}
\end{table}

\section{Per-Game Conformance}

Table~\ref{tab:games} lists every game in the conformance set with its
auto-detected mapper, a short ROM hash, and its per-frame RAM and pixel
exactness, so the headline 64/64 claim can be audited game by game. Both ports
pass every game; \textsc{jaxtari} and \textsc{jutari} agree game-for-game, and
the cross-check between them is part of the test suite.

\section{Throughput}

We measure soft-mode throughput as steps per second (one step $=$ one CPU
instruction) on a fixed ROM after just-in-time warm-up, on a single Apple~M1~Max
core (JAX CPU backend, Julia~1.12), reporting the median of repeated runs.
Table~\ref{tab:s-compare} summarises the three operating points the rest of this
section develops; the per-batch scaling behind the batched rows is in
Tables~\ref{tab:s-batched} (CPU) and~\ref{tab:s-gpu} (GPU).

\begin{table}[tb]
\centering
\caption{Batched soft-mode throughput of \textsc{jaxtari} on the CPU backend
(M1~Max, Pong, a compiled \texttt{lax.scan} \texttt{vmap}-ped over $N$ environments).
\emph{Aggregate} is total environment-steps per second; \emph{per-env} is
aggregate$/N$. Batching trades single-environment latency for aggregate throughput.}
\label{tab:s-batched}
\small
\setlength{\tabcolsep}{5pt}
\begin{tabular}{rrrr}
\toprule
batch $N$ & aggregate & per-env & vs.\ $N{=}1$\\
 & (env-steps/s) & (steps/s) & \\
\midrule
$1$   & $30{,}200$ & $30{,}200$ & $1.0\times$\\
$16$  & $35{,}200$ & $2{,}200$  & $1.2\times$\\
$64$  & $48{,}300$ & $750$      & $1.6\times$\\
$128$ & $55{,}900$ & $440$      & $1.9\times$\\
\bottomrule
\end{tabular}
\end{table}

\begin{table}[tb]
\centering
\caption{Batched soft-mode \textsc{jaxtari} throughput on a single \emph{commodity}
GTX~1080~Ti (11\,GB Pascal), Pong, a $3{,}000$-step \texttt{lax.scan}
\texttt{vmap}-ped over $N$ environments (\textbf{mean $\pm$ standard deviation over
$10$ timed runs}, JIT warm-up excluded). Aggregate environment-steps per second,
forward and forward+gradient. Both peak near $N{=}4096$ at ${\sim}3$M env-steps/s---%
${\sim}50\times$ the CPU \texttt{vmap} asymptote (Table~\ref{tab:s-batched})---and the
gradient costs only ${\sim}5\%$ more. The run-to-run spread is ${<}0.2\%$ throughout,
so the throughput measurement is highly stable.}
\label{tab:s-gpu}
\small
\setlength{\tabcolsep}{6pt}
\begin{tabular}{rrr}
\toprule
batch $N$ & forward & forward+grad\\
 & (env-steps/s) & (env-steps/s)\\
\midrule
$1$     & $2{,}458 \pm 1$         & $1{,}265 \pm 4$\\
$64$    & $258{,}747 \pm 17$      & $224{,}194 \pm 48$\\
$256$   & $677{,}299 \pm 345$     & $595{,}776 \pm 177$\\
$1024$  & $2{,}024{,}825 \pm 216$ & $1{,}819{,}602 \pm 365$\\
$4096$  & $\mathbf{2{,}949{,}204 \pm 2{,}562}$ & $\mathbf{2{,}799{,}834 \pm 162}$\\
$16384$ & $2{,}848{,}876 \pm 490$ & $2{,}575{,}793 \pm 444$\\
$65536$ & $2{,}599{,}829 \pm 100$ & $2{,}534{,}157 \pm 43$\\
\bottomrule
\end{tabular}
\end{table}

At single-instruction granularity \textsc{jutari} sustains ${\sim}370{,}100$
steps/s and \textsc{jaxtari} ${\sim}1{,}178$, a ${\sim}314\times$ gap. The
difference is dominated by JAX's per-operation kernel-launch overhead: each soft
step dispatches one opcode through a $256$-way switch, so at single-instruction
granularity the launch overhead dwarfs the instruction, whereas Julia inlines the
small handlers---a property of the execution granularity, not the AD systems. For a
gradient-based XAI workflow (one forward-plus-backward pass per attribution)
\textsc{jaxtari}'s ${\sim}42$\,s per $50{,}000$-step trace is comfortably
interactive.

Two modes recover throughput beyond this single-step figure: compiling a whole
rollout into one \texttt{lax.scan} (${\sim}30{,}000$ steps/s for a single
environment, a ${\sim}25\times$ gain that amortises the launches), and
\texttt{vmap}-ing that scan over a batch of independent environments.
Table~\ref{tab:s-batched} reports the batched scaling on the CPU
backend (Apple~M1~Max, real Pong ROM, median of repeated $3{,}000$-step runs after
JIT warm-up; the batched run is verified bit-identical to the unbatched one).
Aggregate throughput rises \emph{sub-linearly} with the batch because the
data-dependent $256$-way opcode dispatch (\texttt{lax.switch}) evaluates all
candidate handler branches under \texttt{vmap}; on a CPU these branches serialise,
so the aggregate asymptotes near ${\sim}60{,}000$ env-steps/s (Table~\ref{tab:s-batched}).

\begin{figure}[tb]
\centering
\includegraphics[width=\linewidth]{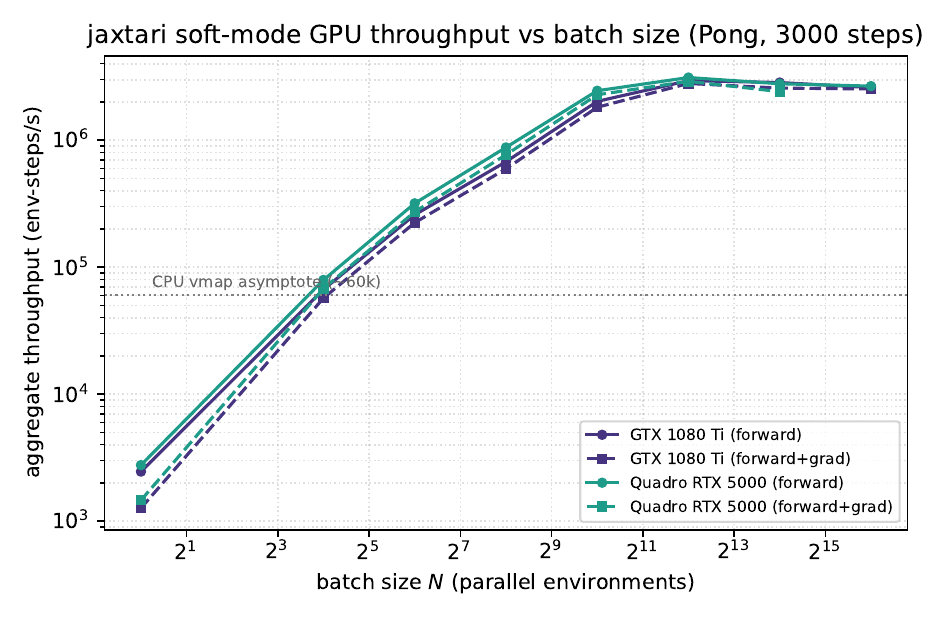}
\caption{\textsc{jaxtari} soft-mode GPU throughput vs.\ batch size $N$ (Pong,
$3{,}000$ steps; two commodity GPUs, GTX~1080~Ti and Quadro~RTX~5000). Forward and forward+gradient both rise
${\sim}1200\times$ from $N{=}1$ to a peak near $N{=}4096$ (${\sim}3$M env-steps/s,
${\sim}50\times$ the CPU \texttt{vmap} asymptote, dotted), then roll off as the device
saturates---the $256$-way \texttt{lax.switch} all-branch dispatch parallelises across
lanes on the GPU where it serialises on the CPU.}
\label{fig:gpu-throughput}
\end{figure}

\paragraph{GPU.} On a GPU the same all-branch evaluation maps onto the device's
parallel lanes instead of serialising, so batching the differentiable rollout is the
intended operating regime. Table~\ref{tab:s-gpu} and Figure~\ref{fig:gpu-throughput}
measure this on a single \emph{commodity} GTX~1080~Ti (an 11\,GB Pascal card; real
Pong ROM, a $3{,}000$-step \texttt{lax.scan} \texttt{vmap}-ped over $N$ environments,
mean over $10$ runs after JIT warm-up---with the per-run standard deviation in
Table~\ref{tab:s-gpu}---verified bit-identical to the CPU rollout).
Aggregate forward throughput climbs from ${\sim}2{,}500$ env-steps/s at $N{=}1$ to a
peak of ${\sim}2.95$M at $N{\approx}4096$---roughly $50\times$ the CPU asymptote and
${\sim}1200\times$ the single-environment rate---then rolls off gently as the device
saturates. The workload is compute/occupancy-bound, not memory-bound: the peak uses
only ${\sim}0.1$\,GB, so the ceiling is lane parallelism rather than capacity.
Crucially the reverse-mode gradient is nearly free at scale---\emph{forward+gradient}
tracks forward to within ${\sim}5\%$ at the peak (${\sim}2.8$M env-steps/s)---so a
commodity GPU turns batched \emph{differentiable} rollouts (gradient-based RL or
attribution over thousands of environments at once) into the natural way to use
\textsc{jaxtari}. A second, newer architecture (Quadro~RTX~5000, 16\,GB Turing)
traces the same curve and peaks slightly higher ($3{,}117{,}908 \pm 2{,}871$
env-steps/s, mean over $10$ runs, at $N{\approx}4096$; Figure~\ref{fig:gpu-throughput}),
so the scaling reflects the
batched all-branch execution rather than one specific GPU. We report widely-available
cards deliberately: the curve's shape, not a datacenter accelerator, is the point.

\section{Implementation Effort}

Figure~\ref{fig:timeline} is the evidence behind the effort estimate in the main
paper. All effort numbers are derived from the project's \texttt{git} history,
which---like the rest of the artifact---will be made public on acceptance of the
paper, so every figure here can be independently recomputed from the commit log.
Because every step of the project was committed, the commit timestamps
record when work actually happened. To turn them into an active-time figure we
group the commits into sessions, treating a gap of more than three hours between
consecutive commits as a session boundary. Three hours is well above the
in-session rhythm---the median gap between consecutive commits is about
$13$ minutes---and well below the multi-hour and overnight gaps that separate
genuine work periods, so the partition is insensitive to the exact threshold:
the active total is $106$ hours at a two-hour cutoff and $162$ hours at a
four-hour cutoff, bracketing the $137$ hours obtained at three hours.

With the three-hour boundary the $373$ commits fall into $52$ sessions whose
durations sum to $136.8$ hours, i.e.\ about $5.7$ round-the-clock days (the
coding agents run continuously, not in eight-hour shifts), even though they are
spread over a $29$-day calendar window. The remaining
${\sim}80\%$ of the calendar span is idle and excluded. It consists mostly of
stretches during which the lead programmer was travelling and without reliable
internet access, which appear in the plot as the long flat segments. This active
total---not the calendar span---is what we compare against the scale of the
reference codebase in the main paper.

\section{Fixed Settings and Parameters}

Table~\ref{tab:s-params} consolidates the fixed settings and relaxation
parameters used across the paper's experiments. The executed gradient mode is
\textsc{soft-ste}, whose forward pass is bit-identical to \textsc{hard} and is
therefore independent of the temperature $T$; the sharpness $\alpha$ and $T$
enter only the relaxation analysis and the surrogate gradient.

\begin{table}[h]
\centering
\caption{Fixed settings and (hyper-)parameters used in the paper's experiments.}
\label{tab:s-params}
\small
\setlength{\tabcolsep}{5pt}
\begin{tabular}{@{}l p{0.52\columnwidth}@{}}
\toprule
Setting & Value\\
\midrule
Soft sharpness $\alpha$       & $2,\,6,\,20$ (gradient-shape study)\\
Soft temperature $T$          & swept; \textsc{soft-ste} forward independent of $T$\\
Numeric representation        & float32, integer-valued in $[-2^{24},2^{24}]$\\
Boot sequence                 & 60 NOOP + 4 RESET frames\\
Episode-start randomization   & 0--30 random NOOPs (Mnih-style)\\
Conformance window            & 30 frames (RAM), 60 frames (screen)\\
Throughput rollout length     & 3{,}000 CPU instructions\\
Throughput batch sizes $N$    & 1, 16, 64, 256, 1024, 4096, 16384, 65536\\
Throughput repeats            & 10 (GPU), 3 (CPU baseline); JIT warm-up excluded\\
\bottomrule
\end{tabular}
\end{table}

\end{document}